\documentclass{article}

\usepackage[english]{babel}

\usepackage[letterpaper,top=2cm,bottom=2cm,left=3cm,right=3cm,marginparwidth=1.75cm]{geometry}

\usepackage{microtype}
\usepackage{graphicx}
\usepackage{subfigure}
\usepackage{booktabs} % for professional tables

\usepackage{hyperref}
\usepackage{bbding}
\usepackage{algorithm}
\usepackage{algorithmic}

\usepackage{amsmath}
\usepackage{amssymb}
\usepackage{mathtools}
\usepackage{amsthm}
\usepackage{authblk}

\usepackage[capitalize,noabbrev]{cleveref}

\theoremstyle{plain}
\newtheorem{theorem}{Theorem}[section]
\newtheorem{proposition}[theorem]{Proposition}
\newtheorem{lemma}[theorem]{Lemma}
\newtheorem{corollary}[theorem]{Corollary}
\theoremstyle{definition}
\newtheorem{definition}[theorem]{Definition}
\newtheorem{assumption}[theorem]{Assumption}
\theoremstyle{remark}
\newtheorem{remark}[theorem]{Remark}

\title{\LARGE \bf
A Stochastic Quasi-Newton Method for Non-convex Optimization \\ with Non-uniform Smoothness
}

\author[1]{Zhenyu Sun} 
\author[1,2]{Ermin Wei}
\affil[1]{ECE, Northwestern University}
\affil[2]{IEMS, Northwestern University}

\date{}

\begin{document}

\maketitle
% \thispagestyle{empty}
% \pagestyle{empty}

%%%%%%%%%%%%%%%%%%%%%%%%%%%%%%%%%%%%%%%%%%%%%%%%%%%%%%%%%%%%%%%%%%%%%%%%%%%%%%%%
\begin{abstract}
Classical convergence analyses for optimization algorithms rely on the widely-adopted uniform smoothness assumption. However, recent experimental studies have demonstrated that many machine learning problems exhibit non-uniform smoothness, meaning the smoothness factor is a function of the model parameter instead of a universal constant. In particular, it has been observed that the smoothness grows with respect to the gradient norm along the training trajectory. Motivated by this phenomenon, the recently introduced $(L_0, L_1)$-smoothness is a more general notion, compared to traditional $L$-smoothness, that captures such positive relationship between smoothness and gradient norm. Under this type of non-uniform smoothness, existing literature has designed stochastic first-order algorithms by utilizing gradient clipping techniques to obtain the optimal $\mathcal{O}(\epsilon^{-3})$ sample complexity for finding an $\epsilon$-approximate first-order stationary solution. Nevertheless, the studies of quasi-Newton methods are still lacking. Considering higher accuracy and more robustness for quasi-Newton methods, in this paper we propose a fast stochastic quasi-Newton method when there exists non-uniformity in smoothness. Leveraging gradient clipping and variance reduction, our algorithm can achieve the best-known $\mathcal{O}(\epsilon^{-3})$ sample complexity and enjoys convergence speedup with simple hyperparameter tuning. Our numerical experiments show that our proposed algorithm outperforms the state-of-the-art approaches.

\end{abstract}

%%%%%%%%%%%%%%%%%%%%%%%%%%%%%%%%%%%%%%%%%%%%%%%%%%%%%%%%%%%%%%%%%%%%%%%%%%%%%%%%
%%%%%%%%%%%%%%%%%%%%%%%%%%%%%%%%%%%%%%%%%%%%%%%%%%%%%%%%%%%%%%%%%%%%%%%%%%%%%%%%%%%%%%%%%%%%%%%%%%%%%%%%%%%%%%%%%%%%%%%%%%%%%%%%%%%%%%%%%%%%%%%%%%%%%%%%%%%%%%%%%%%%%%%%%%%%%%%%%%%%%%%%%%%%%%%%%%%%%%%%%%%%%%%%%%%%%%%%%%%%%%%%%%%%%%%%%%%%%%%%%%%%%%%%%%%%%%%%%%%%%%%%%%%%%%%%%%%%%%%%%%
\section{Introduction}
In this paper, we consider the following stochastic optimization problem:
\begin{equation}    \label{eq_online-obj}
     \min_x ~~ F(x) := \mathbb{E}_{\xi \sim P}[l(x;\xi)]
\end{equation}
where $x \in \mathbb{R}^d$ is the model parameter; $\xi$ is the random variable drawn from some unknown distribution $P$; $l(\cdot;\cdot)$ is the loss function. Note that the expectation is taken over the randomness of data distribution. A well-known example of problem \eqref{eq_online-obj} is \emph{empirical risk minimization} (ERM), where the distribution $P$ is taken as the point mass at each sampled point $\xi_i$, i.e.,
\begin{equation}    \label{eq_finite-sum-obj}
    \min_x ~~ \frac{1}{n}\sum_{i=1}^{n} l(x;\xi_i)
\end{equation}
with $n$ being the sample size. 

In modern machine learning and deep learning contexts, both \eqref{eq_online-obj} and \eqref{eq_finite-sum-obj} may be non-convex due to the structures of data distribution and the embedding of deep neural networks in the loss function, which make them difficult, if not impossible, to solve analytically. In order to solve the aforementioned problem, the idea of employing stochastic approximation (SA) is proposed \cite{robbins1951stochastic}, where stochastic gradient descent (SGD) serves as a classical SA method. Particularly, SGD mimics the gradient descent method by replacing the full gradient with stochastic gradient samples drawn from the dataset. Since only first-order information (i.e., gradients) is utilized, SGD enjoys good simplicity. Then, stochastic first-order methods have been extensively studied and widely used in the deep learning context due to their easy implementation. Techniques of variance reduction have been adopted to achieve low sample complexity \cite{johnson2013accelerating,defazio2014saga,fang2018spider,lei2017non}.

However, these first-order methods may suffer from poor performances in the sense that gradient information fails to capture the curvature properties of objective functions. In contrast, Newton's methods, which leverage the second-order Hessian information, can possibly achieve better accuracy \cite{sohl2014fast,allen2018natasha}. Nevertheless, computing Hessian matrix and its inverse can be computationally expensive, thus motivates the emergence of quasi-Newton methods. In particular, stochastic quasi-Newton (SQN) methods only use first-order information to approximate the Hessian or its inverse such that it can be further integrated with stochastic gradient-based approaches to achieve better performances, compared to first-order methods \cite{byrd2016stochastic,wang2017stochastic}. The appealing advantages of SQN methods, e.g., efficiency, robustness and accuracy \cite{chen2019stochastic}, draw great attention from researchers. Recent studies have shown that by leveraging variance-reduced tools proposed in \cite{fang2018spider}, $\mathcal{O}(\epsilon^{-3})$ sample complexity can be achieved for SQN \cite{zhang2021faster}, which is order-same as the lower bound by first-order methods \cite{arjevani2023lower}. 

Note that for both pure gradient-based methods and SQN methods, existing works all assume the uniform smoothness condition on the objective function for convergence analysis, i.e. $\nabla F(x)$ is $L$-Lipschitz continuous with universal constant $L$ for any model parameter $x$. However, recent experimental evidence has revealed that the Lipschitz constant of the objective smoothness grows in the gradient norm along the training trajectory \cite{zhang2019gradient}. Figure \ref{fig_clippSmoothness} illustrates this fact. This non-uniform smoothness indicates that many important objective functions fail to exhibit a universal upper bound on the smoothness constants and hence many existing convergence analyses have limited applicability in practical settings. Inspired by this phenomenon, \cite{zhang2019gradient} introduces a more general notion of smoothness named {\it $(L_0, L_1)$-smoothness}, where the smoothness increases linearly in the gradient norm, i.e., $\Vert \nabla^2 F(x)\Vert \le L_0 + L_1 \Vert \nabla F(x) \Vert$. Later in \cite{mei2021leveraging}, it has been  shown analytically that some reinforcement learning (RL) problems are characterized by this particular type of smoothness. 

\begin{figure}[ht]
\vskip 0.2in
\begin{center}
\centerline{\includegraphics[width=6.5cm]{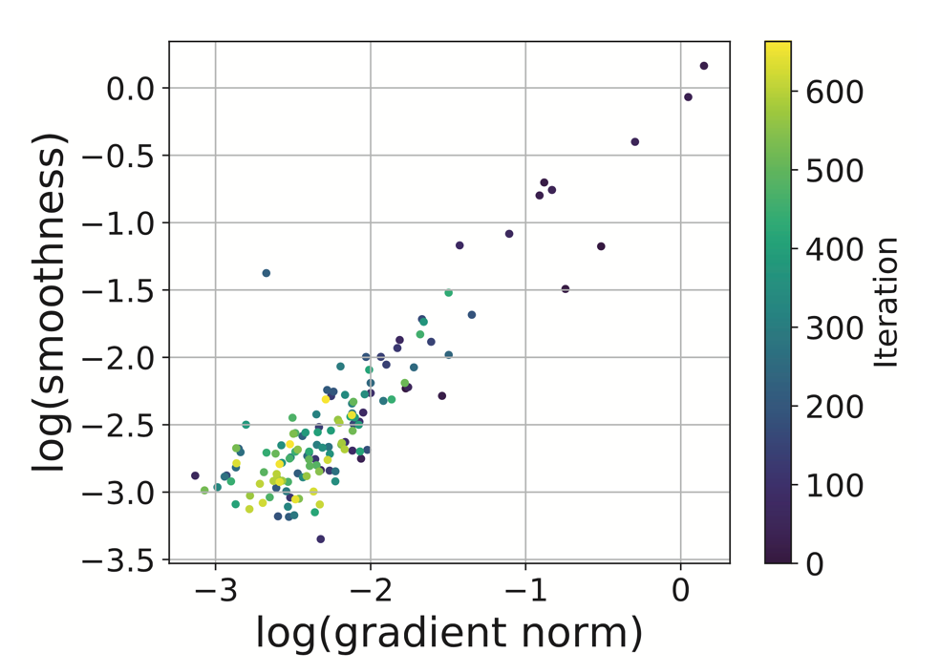}}
\caption{Smoothness increases with gradient norm along the training trajectory (figure taken from \cite{zhang2019gradient}}
\label{fig_clippSmoothness}
\end{center}
\vskip -0.2in
\end{figure}

As a result of non-uniform smoothness, gradients can grow rapidly, gradient clipping techniques have been applied to mitigate this effect. Under the stochastic setting, \cite{zhang2019gradient} shows that clipped SGD can find an $\epsilon$-approximate first-order stationary solution with $\mathcal{O}(\epsilon^{-4})$ samples. Further combining the idea of variance reduction, \cite{reisizadeh2023variance} proposes $(L_0, L_1)$-Spider that achieves the optimal $\mathcal{O}(\epsilon^{-3})$ sample complexity. However, SQN methods are still lacking, especially when non-uniform smoothness exists. Therefore, in order to bridge this gap, we aim to design a SQN method that captures objective's non-uniform smoothness, while obtaining low sample complexity.

\subsection{Related work}
\paragraph{SGD-based methods in non-convex optimization}
Stochastic first-order algorithms are widely used and well-investigated in modern machine learning tasks. For non-convex optimization, the classical SGD algorithm has been proven to achieve an overall $\mathcal{O}(\epsilon^{-4})$ samples to find an $\epsilon$-approximate first-order stationary point \cite{ghadimi2013stochastic,ghadimi2016accelerated}. Later, several SGD variants equipped with variance reduction techniques emerge, including SAGA \cite{defazio2014saga}, SVRG \cite{johnson2013accelerating}, SARAH \cite{nguyen2017sarah} etc. Due to the effectiveness of variance reduction, it is shown that SVRG improves the sample complexity to $\tilde{\mathcal{O}}(\epsilon^{-10/3})$ for non-convex optimization \cite{allen2016variance,reddi2016stochastic}. Recently, \cite{fang2018spider} proposed Spider, a less costly approach that tracks the true gradients in an iterative way. Spider improves the sample complexity to $\mathcal{O}(\epsilon^{-3})$, which matches the lower bound complexity shown in \cite{arjevani2023lower}. This means Spider is an order-optimal algorithm to find stationary points of non-convex smooth functions. A nested variance-reduced algorithm in \cite{zhou2020stochastic} achieves similar complexity bound. Other works with optimal complexity can be found therein \cite{pham2020proxsarah,li2021page,li2021zerosarah}.

\paragraph{SQN methods in non-convex optimization}
Newton's methods usually have rapid convergence speed due to the use of Hessian information \cite{boyd2004convex,moritz2016linearly}, which have recently been applied to non-convex optimization problems \cite{kohler2017sub,zhou2018stochastic}. However, prohibitively expensive computation cost in calculating Hessian matrix and its inverse hinders the application of Newton's methods to large-scale machine learning problems. Quasi-Newton methods then serve as effective candidates to address this computation challenge. SQN methods are extensively studied for both strongly convex and convex optimization problems \cite{mokhtari2014res,schraudolph2007stochastic}. Variance reduction was also embedded with L-BFGS method recently \cite{moritz2016linearly,lucchi2015variance}. In terms of non-convex optimization, SdLBFGS was proposed in \cite{wang2017stochastic} where $\mathcal{O}(\epsilon^{-4})$ complexity is achieved. Inspired by the variance reduction idea as in Spider, \cite{zhang2021faster} improved the sample complexity bound to $\mathcal{O}(\epsilon^{-3})$. 

\paragraph{Gradient clipping and non-uniform smoothness}
Gradient clipping has been widely used in training deep learning models to mitigate issues related to very large gradients  \cite{merity2017regularizing,gehring2017convolutional,peters2019knowledge}. The work \cite{zhang2019gradient} lays out a theoretical basis to better understand the superior performances of clipping-based algorithms. Due to the positive relation between smoothness and gradient norm as shown in Figure \ref{fig_clippSmoothness}, \cite{zhang2019gradient} introduced the notion of $(L_0, L_1)$-smoothness, which extends the conventional $L$-smoothness. \cite{mei2021leveraging} then introduced a more general smoothness notion and proved that any RL problem with a finite Markov decision process (MDP) satisfies $(L_0, L_1)$-smoothness. Under $(L_0, L_1)$-smoothness, \cite{qian2021understanding} studies clipping for incremental gradient-based methods. Convergence analysis on the stochastic normalized gradient descent method is provided by \cite{zhao2021convergence} for $(L_0, L_1)$-smooth non-convex optimization. Very recently, \cite{reisizadeh2023variance} combines gradient clipping with variance reduction techniques to show the optimal $\mathcal{O}(\epsilon^{-3})$ sample complexity under $(L_0, L_1)$-smoothness. Apart from optimization literature, this particular type of smoothness has been studied for variational inference problems in \cite{sun2023convergence}.

\subsection{Our contributions}
We summarize our main contributions as follows: (1) We propose a general stochastic quasi-Newton framework for non-convex optimization with $(L_0, L_1)$-smoothness, which works with any Hessian inverse approximation that is positive definite. Convergence analysis and sample complexity are given, where the complexity bound is $\mathcal{O}(\epsilon^{-3}\lambda_M^2 / \lambda_m^2)$ with $\lambda_m$ and $\lambda_M$ being the smallest and largest eigenvalues of the  Hessian inverse approximation. (2) We then design an adaptive L-BFGS-based algorithm that guarantees the positive definiteness of the generated Hessian inverse approximation. Moreover, by  tuning the design parameters, we can control these eigenvalues, which further allows us to control the sample complexity. Table \ref{tab:comparison} provides a clear comparison of our algorithm and other state-of-the-art methods, where Clip represents "clipping"; VR represents "variance reduction". To the best of our knowledge, our approach is the first stochastic quasi-Newton method that achieves the best-known $\mathcal{O}(\epsilon^{-3})$ sample complexity to find an $\epsilon$-approximate first-order stationary solution when the objective function is not uniformly smooth.

\begin{table*}
    \centering
    \caption{Sample complexity and smoothness for stochastic non-convex optimization algorithms}
    \begin{tabular}{p{10em}p{8em}p{8em}p{8em}p{10em}}
         \toprule
         Reference  &  Algorithm type  & Technique &  Complexity  &  Smoothness  \\
         \midrule
         SGD \cite{ghadimi2013stochastic}  &   first-order   &  \XSolidBrush &$\mathcal{O}(\epsilon^{-4})$  &  $L$-smooth  \\
         \midrule
         SVRG \cite{allen2016variance} & first-order & VR & $\tilde{\mathcal{O}}(\epsilon^{-10/3})$ & $L$-smooth \\
         \midrule
         ClippedSGD \cite{zhang2019gradient}  & first-order  & Clip &  $\mathcal{O}(\epsilon^{-4})$  & $(L_0, L_1)$-smooth  \\
         \midrule
         Spider \cite{fang2018spider} &  first-order  & VR &  $\mathcal{O}(\epsilon^{-3})$  &  $L$-smooth   \\
         \midrule
          $(L_0, L_1)$-Spider \cite{reisizadeh2023variance} &  first-order  & Clip; VR &  $\mathcal{O}(\epsilon^{-3})$  & $(L_0, L_1)$-smooth  \\
         \midrule
         SdLBFGS \cite{wang2017stochastic} &  quasi-Newton  & \XSolidBrush &  $\mathcal{O}(\epsilon^{-4})$  &  $L$-smooth   \\
         \midrule
         SSQNMED \cite{zhang2021faster} &  quasi-Newton  & VR &  $\mathcal{O}(\epsilon^{-3})$  &  $L$-smooth   \\
         \midrule
         ClippedSQN (ours)  &  quasi-Newton  & Clip; VR &  $\mathcal{O}(\epsilon^{-3})$  &  $(L_0, L_1)$-smooth   \\
         \bottomrule
    \end{tabular}
    \label{tab:comparison}
\end{table*}

%%%%%%%%%%%%%%%%%%%%%%%%%%%%%%%%%%%%%%%%%%%%%%%%%%%%%%%%%%%%%%%%%%%%%%%%%%%%%%%%%%%%%%%%%%%%%%%%%%%%%%%%%%%%%%%%%%%%%%%%%%%%%%%%%%%%%%%%%%%%%%%%%%%%%%%%%%%%%%%%%%%%%%%%%%%%%%%%%%%%%%%%%%%%%%%%%%%%%%%%%%%%%%%%%%%%%%%%%%%%%%%%%%%%%%%%%%%%%%%%%%%%%%%%%%%%%%%%%%%%%%%%%%%%%%%%%%%%%%%%%%
\section{Preliminaries} \label{sec_pre}
In this section, we describe the goal of our algorithm in terms of an optimality condition and also present a condition of non-uniform smoothness, under which our proposed method has performance guarantees.

\subsection{Optimality condition}
Noting that finding the global minimum for general stochastic non-convex optimization problems is NP-hard \cite{hillar2013most}, in this work we instead focus on finding an \emph{$\epsilon$-approximate first-order stationary solution}, which is a standard performance measure for non-convex optimization algorithms. Formally, this optimality condition is defined as follows:
\begin{definition}
    Given some algorithm $\mathcal{A}$, let $\tilde{x}$ be an output of algorithm $\mathcal{A}$. Then $\tilde{x}$ is said to be an \emph{$\epsilon$-approximate first-order stationary solution} of \eqref{eq_online-obj} if $\mathbb{E}\Vert \nabla F(\tilde{x}) \Vert \le \epsilon$ for any $\epsilon > 0$.
\end{definition}
We note that the expectation in $\mathbb{E}\Vert \nabla F(\tilde{x}) \Vert$ is taken over the randomness of $\tilde{x}$ because  $\mathcal{A}$ is a randomized algorithm.

\subsection{$(L_0, L_1)$-smoothness}\label{sec:non-uniform_smooth_def}
% As introduced in \cite{mei2021leveraging}, we provide the generalized notion of smoothness, which captures the non-uniform property of parameter changing on the smoothness factor.
% \begin{definition}\cite{mei2021leveraging}  \label{def_beta-smooth}
%     The differentiable function $F : \mathbb{R}^d \to \mathbb{R}$ is $\alpha(x)$-smooth if for any $x, y \in \mathbb{R}^d$, 
%     \begin{equation}
%         \Vert \nabla F(x) - \nabla F(y) \Vert \le \alpha(x) \Vert x - y \Vert
%     \end{equation}
%     where $\alpha$ is a positive valued function.
% \end{definition}
% 
We first provide a more general notion of smoothness, i.e., $(L_0, L_1)$-smoothness introduced in \cite{zhang2019gradient,reisizadeh2023variance}, which captures the non-uniform property of the smoothness factor. It is formally described by the following definition:
\begin{definition} \cite{reisizadeh2023variance}  \label{def_L0L1-smooth}
    The differentiable function $F : \mathbb{R}^d \to \mathbb{R}$ is $(L_0, L_1)$-smooth if for any $x, y \in \mathbb{R}^d$, 
    \begin{equation}\label{ineq:general_smooth}
        \Vert \nabla F(x) - \nabla F(y) \Vert \le (L_0 + L_1 \Vert \nabla F(x) \Vert) \Vert x - y \Vert
    \end{equation}
    where $L_0 > 0$ and $L_1 \ge 0$ are two constants.
\end{definition}
Clearly, $(L_0, L_1)$-smoothness notion generalizes the traditional Lipschitz-gradient condition, since any function $F$ with Lipschitz gradient satisfying $\Vert \nabla F(x) - \nabla F(y) \Vert \le L\Vert x - y \Vert$ for any $x$ and $y$ satisfies \eqref{ineq:general_smooth} with $L_0=L$ and $L_1=0$.

\subsection{Examples with $(L_0, L_1)$-smooth properties}
For the rest of this paper, we focus on the objective functions with $(L_0, L_1)$-smoothness shown in Definition \ref{def_L0L1-smooth}. This is motivated by the observation that many machine learning problems have $(L_0, L_1)$-smoothness. We next illustrate this by two examples.

\begin{proposition} \label{prop_cross-entropy-smooth}
    Consider $F(x) = y \log(\hat{y})$, where $\hat{y} = \sigma(u^T x)$ with $\sigma(\cdot)$ being the sigmoid function and $y, u$ are constant scalars or vectors with suitable dimensions. Then, $F(x)$ is $(L_0, \Vert u \Vert)$-smooth for any $L_0 > 0$.
\end{proposition}
% \begin{proof}
%     A straightforward calculation gives
%     $$
%         \nabla F(x) = \frac{y}{\hat{y}} \hat{y}(1 - \hat{y}) u = y (1 - \hat{y}) u
%     $$
%     $$
%         \nabla^2_{xx} F(x) = - y \hat{y}(1 - \hat{y}) u u^T .
%     $$
%     Thus,
%     $$
%         \Vert \nabla^2 F(x) \Vert = \hat{y} \Vert u \Vert \Vert \nabla F(x) \Vert \le \Vert u \Vert \Vert \nabla F(x) \Vert
%     $$
%     by noting $\hat{y} \in (0, 1)$.
% \end{proof}
Function $F$ in the above proposition is commonly used in classification tasks with cross-entropy loss \cite{yamashita2018convolutional}. More generally, if $x$ is the parameter of the last layer of the designed neural network and $u$ is the output of the previous layer, we observe from Proposition \ref{prop_cross-entropy-smooth} that the Hessian and the gradient are positively related, which also demonstrates the validality of $(L_0, L_1)$-smoothness in deep learning context.

In addition, as shown by \cite{mei2021leveraging}, in the reinforcement learning (RL) problem with finite Markov Decision Processes (MDPs), $(L_0, L_1)$-smoothness is also satisfied by the value function:
\begin{proposition}\cite{mei2021leveraging}
    For any finite MDP, let $V^{\pi_{\theta}}(s_0)$ be the value function, where $\pi_{\theta}$ and $s_0$ are the policy parametrized by $\theta$ and the initial state, respectively. Then, $V^{\pi_{\theta}}(s_0)$ is $(L_0, L_1)$-smooth for any $L_0 > 0$ and some $L_1 > 0$ determined by problem related constants. 
\end{proposition}

% Figure \ref{fig_rlSmootness} presents the spectral radius of the Hessian matrix and the policy gradient norm along the training trajectory for a single-state MDP, which demonstrates the validality of $(L_0, L_1)$-smoothness.

% \begin{figure}[ht]
% \begin{center}
% \centerline{\includegraphics[width=5.5cm]{icml_figures/rl_smoothness.eps}}
% \caption{Hessian spectral radius and policy gradient norm (figure taken from \cite{mei2021leveraging})}
% \label{fig_rlSmootness}
% \end{center}
% \end{figure}

%%%%%%%%%%%%%%%%%%%%%%%%%%%%%%%%%%%%%%%%%%%%%%%%%%%%%%%%%%%%%%%%%%%%%%%%%%%%%%%%%%%%%%%%%%%%%%%%%%%%%%%%%%%%%%%%%%%%%%%%%%%%%%%%%%%%%%%%%%%%%%%%%%%%%%%%%%%%%%%%%%%%%%%%%%%%%%%%%%%%%%%%%%%%
\section{A Clipped Stochastic Quasi-Newton Method}    \label{sec_clippSQN}
%As stated, many learning problems maintain objective functions that satisfy $(L_0, L_1)$-smoothness. Thus, 
In this section we propose a general stochastic quasi-Newton (SQN) framework with clipping techniques that yields fast convergence and low sample complexity for the particular type of non-uniformly smooth functions defined in  Section \ref{sec:non-uniform_smooth_def}. We first present some assumptions on the objective property and sampling process.
\begin{assumption}[Averaged $(L_0, L_1)$-Lipschitz smoothness]   \label{assump_L0L1-smooth}
    Suppose the loss function $l(x;\cdot)$ is differentiable in $x$ for any $x$. Then $F$ is averaged $(L_0, L_1)$-smooth if for any $x, y$, there exist some constant $L_0 > 0$ and $L_1 \ge 0$ such that
    \begin{equation}
        \mathbb{E}\Vert \nabla l(x;\xi) - \nabla l(y;\xi) \Vert^2 \le (L_0 + L_1\Vert \nabla F(x) \Vert)^2\Vert x - y \Vert^2 . \nonumber
    \end{equation}
\end{assumption}

\begin{assumption}[Unbiased estimate and bouned variance]   \label{assump_bnd-var}
    For any $x$, we have 
    \begin{align*}  
        &\mathbb{E}[\nabla l(x; \xi)] = \nabla F(x)  \\
        &\mathbb{E}\Vert \nabla l(x; \xi) - \nabla F(x) \Vert^2 \le \sigma^2 < \infty .
    \end{align*}
\end{assumption}

Our proposed method relies on building an approximation of  Hessian inverse, $H_k$, at each iteration, which satisfies the following assumptions.
\begin{assumption}  \label{assump_pos-H}
    There exist some strictly positive constants $\lambda_m$ and $\lambda_M$ such that for any $k = 0,\dots,K$, we have
    $$
        \lambda_m I \preceq H_k \preceq \lambda_M I.
    $$
\end{assumption}

\begin{assumption}  \label{assump_H-dependence}
    For any $x_k$ and $v_k$ with  $k \ge 1$ , $H_k$ depends only on $\xi_{[k-1]}$, where $\xi_{[k]} := (\xi_0, \dots, \xi_k)$ which is all the samples from iteration $0$ to iteration $k$. Or equivalently, it means that
    $$
        \mathbb{E}[H_k v_k | \xi_{[k-1]}, x_k] = H_k \mathbb{E}[v_k | x_k].
    $$
\end{assumption}
It is worth noting that Assumption \ref{assump_L0L1-smooth} implies $(L_0, L_1)$-smoothness on $F$ by 
$
    \Vert \nabla F(x) - \nabla F(y) \Vert^2 \le \mathbb{E}\Vert \nabla l(x;\xi) - \nabla l(y;\xi) \Vert^2.
$

We  first assume conditions in Assumptions \ref{assump_pos-H} and \ref{assump_H-dependence} are met and defer their justification to the next section where we describe a specific algorithm to generate $H_k$. 
%Before proposing a specific algorithm to generate $H_k$ that satisfies Assumptions \ref{assump_pos-H} and \ref{assump_H-dependence}, 
We next present Algorithm \ref{alg_clippSQN}, a general SQN framework that obtains fast convergence and $\mathcal{O}(\epsilon^{-3})$ sample complexity to achieve $\epsilon$-approximate first-order solutions for any stochastic non-convex optimization problem with averaged $(L_0, L_1)$-smoothness. 
\begin{algorithm}[tb]
   \caption{Clipped Stochastic Quasi-Newton Method}
   \label{alg_clippSQN}
\begin{algorithmic}[1]
   \STATE {\bfseries Input:} initial point $x_0$, positive definite $H_0$, batch size $|\mathcal{S}_1|$ and $|\mathcal{S}_2|$, integers $K$ and $r$, stepsizes $\{ \eta_k \}$
   % \REPEAT
   \FOR {$k=0,1,\dots,K-1$}      
            \IF{$k$ $\mathrm{mod}$ $r$ = 0}
            \STATE Draw samples $\mathcal{S}_1$ and compute $v_k = \nabla l(x_k; \mathcal{S}_1)$.
            \ELSE
            \STATE Draw samples $\mathcal{S}_2$ and compute $v_k = v_{k-1} + \nabla l(x_k; \mathcal{S}_2) - \nabla l(x_{k-1}; \mathcal{S}_2)$.
            \ENDIF
            \STATE Calculate stepsize $\eta_k$ from \eqref{eq_stepsize}.
            \STATE Generate a positive definite Hessian inverse approximation $H_k$ at $x_k$ and calculate $H_k v_k$ according to Algorithm \ref{alg_AdaL-BFGS}.
            \STATE $x_{k+1} = x_k - \eta_k H_k v_k$.
    \ENDFOR
   % \UNTIL{$noChange$ is $true$}
   \STATE {\bfseries Output:} $\tilde{x}_K$ sampled uniformly from $\{ x_k \}_{k=0}^{K-1}$
\end{algorithmic}
\end{algorithm}

Algorithm \ref{alg_clippSQN} is inspired by Spider \cite{fang2018spider}. In particular, Spider maintains an estimate of the true gradient $\nabla F(x_k)$ denoted by $v_k$ and updates $v_k$ by utilizing mini-batches $\mathcal{S}_1$ and $\mathcal{S}_2$ to reduce the variance. Then, assuming $L$-smooth $F$, the model parameter is updated by $x_{k+1} = x_k - \eta_k v_k$ with $\eta_k = \min \{ 1 / (2L), \epsilon / (L\Vert v_k \Vert) \}$. However, Spider cannot deal with non-uniformly smooth objective functions in the sense that there is no uniform upper bound for $\Vert \nabla^2 F(x) \Vert$. If the smoothness parameter grows with respect to the gradient norm, stepsize $\eta_k$ will keep decreasing and finally vanish when the norm of the gradient is extremely large, which means $x_k$ could get stuck at a non-stationary point in our setting.

In order to deal with the stepsize-vanishing problem caused by non-uniform smoothness, \cite{reisizadeh2023variance}  proposes $(L_0, L_1)$-Spider for $(L_0, L_1)$-smooth functions. In \cite{reisizadeh2023variance}, the authors modify the stepsize as $\eta_k = \min \{ 1 / (2L_0), \epsilon / (L_0\Vert v_k \Vert), \epsilon / (L_1 \Vert v_k \Vert^2) \}$ and retain the update $x_{k+1} = x_k - \eta_k v_k$. However, both Spider and $(L_0, L_1)$-Spider are first-order methods, which compared to quasi-Newton methods are usually slower and less robust \cite{wang2017stochastic}. Adopting the potential speed advantage of quasi-Newton methods to non-uniform smooth objective functions is main contribution of proposing Algorithm \ref{alg_clippSQN}.

Note that for now we use $H_k v_k$ rather than $v_k$ to update our model, which introduces additional challenges for stepsize design and convergence analysis. Next theorem addresses these challenges and characterizes an upper bound on the sample complexity of finding an $\epsilon$-approximate stationary solution for Algorithm \ref{alg_clippSQN}. Denoting $\Delta_0 := \mathbb{E}[F(x_0) - F^*]$ with $F^* := \min_x F(x) > -\infty$, we have: 

\begin{theorem} \label{thm_convergence}
    Suppose Assumptions \ref{assump_L0L1-smooth}-\ref{assump_H-dependence} hold. For any $0 < \beta \le \frac{\lambda_m}{1 + \lambda_M^2}$, $0 < c \le \frac{4\lambda_m - 2\beta(1 + \lambda_M^2)}{L_1 \lambda_M^2 \beta (3 + \beta^2)}$, defining $h_1^{\beta}(c) = \lambda_m - \frac{\beta \lambda_M^2}{4}(2 + 3L_1 c)$ and setting $|\mathcal{S}_1| = \frac{2\sigma^2}{\epsilon^2}$, $|\mathcal{S}_2| = \frac{4 h_1^{\beta}(c)^2}{\epsilon}$, $r = \frac{1}{\epsilon}$ with stepsizes selected by
    \begin{equation}    \label{eq_stepsize}
        \eta_k = \min \left\{\frac{h_1^{\beta}(c)}{2 L_0 \lambda_M^2}, \frac{h_1^{\beta}(c) \epsilon}{L_0 \lambda_M^2 \Vert v_k \Vert},  \frac{h_1^{\beta}(c) \epsilon}{L_1 \lambda_M^2 \Vert v_k \Vert^2} \right\} ,
    \end{equation}
    then we can achieve
    $$
        \mathbb{E}\Vert \nabla F(\tilde{x}_K) \Vert = \mathcal{O}((1 + \beta^{-2})\epsilon)
    $$
    for $\tilde{x}_K$ being the output of Algorithm \ref{alg_clippSQN} by running $K = \lceil \frac{2 L_0 \lambda_M^2 \Delta_0}{h_1^{\beta}(c)^2 \epsilon^2} \rceil$ iterations. Moreover, the total sample complexity is with the order of $\mathcal{O}(\epsilon^{-3} \lambda_M^2 / \lambda_m^2)$.
\end{theorem}

\begin{remark}
    In Theorem \ref{thm_convergence}, when $0 < \beta \le \frac{\lambda_m}{1 + \lambda_M^2}$, the upper bound for $c$ is positive. Moreover, for $c$ chosen as stated, one can easily check that $0 < \beta^2 \big( \frac{1}{2\beta} + \frac{L_1 \beta \lambda_M^2 c}{4} \big) \le h_1^{\beta}(c)$, indicating the proper choice for $\eta_k$.
\end{remark}

%%%%%%%%%%%%%%%%%%%%%%%%%%%%%%%%%%%%%%%%%%%%%%%%%%%%%%%%%%%%%%%%%%%%%%%%%%%%%%%%%%%%%%%%%%%%%%%%%%%%%%%%%%%%%%%%%%%%%%%%%%%%%%%%%%%%%%%%%%%%%%%%%%%%%%%%%%%%%%%%%%%%%%%%%%%%%%%%%%%%%%%%%%%%
\section{Generating $H_k$ with Controllable $(\lambda_m, \lambda_M)$}   \label{sec_generate-H}
The convergence result in Theorem \ref{thm_convergence} is established based on Assumptions \ref{assump_pos-H} and \ref{assump_H-dependence}, both of which we yet need to justify. In this section, we propose one  approach to generating $H_k$ that satisfies these two assumptions. 

\subsection{Stochastic adaptive BFGS method}
We first propose an adaptive BFGS method that outputs $H_k$ satisfying Assumptions \ref{assump_pos-H} and \ref{assump_H-dependence}, where certain design parameters can be adjusted base on problem features. Define $\bar{g}_k := \nabla l(x_k ; \mathcal{S}_{k-1}) = \frac{1}{|\mathcal{S}_{k-1}|} \sum_{i \in \mathcal{S}_{k-1}} \nabla l(x_{k}; \xi_{i, k-1})$ and $g_k := \nabla l(x_k ; \mathcal{S}_{k}) = \frac{1}{|\mathcal{S}_{k}|} \sum_{i \in \mathcal{S}_{k}} \nabla l(x_{k}; \xi_{i, k})$, where $\mathcal{S}_k$ is the batch of samples drawn at iteration $k$ and $\xi_{i,k}$ is the $i$-th sample in $\mathcal{S}_k$. And let $y_{k-1} := \bar{g}_k - g_{k-1}$, $s_{k-1} := x_k - x_{k-1}$. Then we define
\begin{equation}    \label{eq_yk-hat}
    \hat{y}_{k-1} := \hat{w}_{k-1} \big( \hat{\theta}_{k-1} y_{k-1} + (1 - \hat{\theta}_{k-1}) B_{k-1} s_{k-1} \big)
\end{equation}
where 
\begin{equation}
    \hat{\theta}_{k-1} = \left\{ 
    \begin{array}{ccc}
         \frac{(1 - \hat{q}_{k-1}) \mu_{k-1}}{\mu_{k-1} - s_{k-1}^T y_{k-1}} & , & s_{k-1}^T y_{k-1} < \hat{q}_{k-1} \mu_{k-1}  \\
         & ~ &  \\
         1 & , & \mathrm{otherwise}
    \end{array}   
    \right .    \nonumber
\end{equation}
with $0< \hat{q}_{k-1} < 1$, $\hat{w}_{k-1} > 0$ being some design parameters, and denoting $\mu_{k-1} := s_{k-1}^T B_{k-1} s_{k-1}$. Note that if $B_{k-1} \succ 0$, then it is straightforward to verify $0 < \hat{\theta}_{k-1} \le 1$. The adaptive BFGS method updates $B_k$ as follows:
\begin{equation}    \label{eq_update-Bk}
    B_k = B_{k-1} - \frac{B_{k-1} s_{k-1} s_{k-1}^T B_{k-1}}{s_{k-1}^T B_{k-1} s_{k-1}} + \frac{\hat{y}_{k-1} \hat{y}_{k-1}^T}{s_{k-1}^T \hat{y}_{k-1}} .
\end{equation}
According to the matrix inverse theorem, the update of $H_k = B_k^{-1}$ is equivalently stated as
\begin{align}    \label{eq_update-Hk}
    H_k = & (I - \hat{\rho}_{k-1} s_{k-1} \hat{y}_{k-1}^T) H_{k-1} (I - \hat{\rho}_{k-1} \hat{y}_{k-1} s_{k-1}^T)     \nonumber   \\
    & + \hat{\rho}_{k-1} s_{k-1} s_{k-1}^T
\end{align}
where $\hat{\rho}_{k-1} = (s_{k-1}^T \hat{y}_{k-1})^{-1}$.

Note that both updates \eqref{eq_update-Bk} and \eqref{eq_update-Hk} satisfy Assumption \ref{assump_H-dependence}. Our update is motivated by  \cite{wang2017stochastic} and modified by introducing the parameters $\hat{q}_{k-1}$ and $\hat{w}_{k-1}$. Careful selection of these parameters enables us to  tune $\lambda_m$ and $\lambda_M$ in Assumption \ref{assump_pos-H}, which leads to some control in sample complexity and in turn convergence speed in Theorem \ref{thm_convergence}. This distinguishes our adaptive BFGS method from \cite{wang2017stochastic}. We will defer the analysis on $\lambda_m$ and $\lambda_M$ to Section \ref{subsec_adp-LBFGS}, where a computation light version of adaptive BFGS method is proposed. Here, we present the following lemma to show that our stochastic adaptive BFGS method guarantees positive definiteness of $B_k$ and $H_k$, hence ensuring Assumption \ref{assump_pos-H} is satisfied. 

\begin{lemma}   \label{lmm_pd-Hk}
    Considering the updates \eqref{eq_update-Bk} and \eqref{eq_update-Hk}, then $s_{k-1}^T \hat{y}_{k-1} \ge \hat{w}_{k-1} \hat{q}_{k-1} s_{k-1}^T B_{k-1} s_{k-1}$. And if $B_{0} \succ 0$, $B_k$ and $H_k$ are positive definite for all $k \ge 1$.
\end{lemma}

\subsection{Stochastic adaptive L-BFGS method}  \label{subsec_adp-LBFGS}
In the previous section, computing $H_k$ by the stochastic adaptive BFGS method requires $\mathcal{O}(d^2)$ number of operations. It is computationally expensive especially when the dimension of the model $x_k$ is large, which is usually the case for deep neural networks. We
leverage the idea of the L-BFGS method first proposed by \cite{liu1989limited} to reduce the computation cost. 

At iteration $k$, we maintain a memory with size $p$ to store two sequences $\{s_j\}$ and $\{y_j\}$ for $j = k-p, \dots, k-1$. Given the initial $H_{k,0}$, keep updating $H_{k,i}$ with $i = p+1+j - k$ by utilizing $\{s_j\}$ and $\{y_j\}$ in the memory for $p$ times. Finally, output $H_{k,p}$ as the approximation of the Hessian inverse at $x_k$. 

Specifically, the initial $H_{k,0}$ is given by
\begin{equation}    \label{eq_H_k,0}
    H_{k,0} = c_k^{-1} I, \mathrm{with} ~ c_k = \max \left\{ \delta, w_{k-1} \frac{y_{k-1}^T y_{k-1}}{s_{k-1}^T y_{k-1}} \right\}
\end{equation}
where $\delta > 0$ and $w_{k-1} > 0$ are design parameters.

Then similar to \eqref{eq_yk-hat}, we define
\begin{equation}    \label{eq_yk-bar}
    \bar{y}_{k-1} = w_{k-1} \big( \theta_{k-1} y_{k-1} + (1 - \theta_{k-1}) H_{k,0}^{-1} s_{k-1} \big)
\end{equation}
where 
\begin{equation}  
    \theta_{k-1} = \left\{ 
    \begin{array}{ccc}
         \frac{(1 - q_{k-1}) \bar{\mu}_{k-1}}{\bar{\mu}_{k-1} - s_{k-1}^T y_{k-1}} & , & s_{k-1}^T y_{k-1} < q_{k-1} \bar{\mu}_{k-1}  \\
         & ~ &  \\
         1 & , & \mathrm{otherwise}
    \end{array}   
    \right .    \nonumber
\end{equation}
with $\bar{\mu}_{k-1} := s_{k-1}^T H_{k,0}^{-1} s_{k-1}$ and some designed $0 < q_{k-1} < 1$.

Consider the update of $H_{k,i}$ as follows: for any $i = 1,\dots,p$, $j = k - (p - i + 1)$,
\begin{equation}    \label{eq_H_k,i}
    H_{k,i} = (I - \rho_j s_j \bar{y}_j^T) H_{k, i-1} (I - \rho_j \bar{y}_j s_j^T) + \rho_j s_j s_j^T, 
\end{equation}
where $\rho_j = (s_j^T \bar{y}_j)^{-1}$. Finally, generate $H_k = H_{k,p}$ as the Hessian inverse approximation stated in Algorithm \ref{alg_clippSQN}.

Similar to Lemma \ref{lmm_pd-Hk} we have positive definite $H_{k,i}$ for any $k$ and $i$:
\begin{lemma}   \label{lmm_H_k,i-pd}
    Considering the update \eqref{eq_H_k,i} with $H_{k,0}$ defined by \eqref{eq_H_k,0}, then $s_k^T \bar{y}_k \ge w_k q_k s_k^T H_{k+1,0} s_k > 0, \forall k \ge 0$ and $H_{k,i} \succ 0, \forall i=1,\dots, p$. 
\end{lemma}

Algorithm \ref{alg_AdaL-BFGS} summarizes a practical way to obtain $H_k v_k$ rather than $H_k$, which essentially achieves lower computational complexity. In particular, note that two loops in Algorithm \ref{alg_AdaL-BFGS} involve $p$ scalar-vector multiplications and $p$ vector inner products, hence indicating the total number of multiplications in Algorithm \ref{alg_AdaL-BFGS} is with the order of $\mathcal{O}(pd)$. As suggested in \cite{nocedal1999numerical}, $p$ can be chosen to be much smaller than $d$ in high-dimensional cases. Thus, our Adaptive L-BFGS method enjoys computational efficiency in modern neural-network based optimization.

\begin{algorithm}[tb]
   \caption{Stochastic Adaptive L-BFGS Method}
   \label{alg_AdaL-BFGS}
\begin{algorithmic}[1]
   \STATE {\bfseries Input:} memory size $p$, positive scalars $\delta, q, \kappa$, model parameters $x_k$ and $x_{k-1}$, mini-batch samples $\mathcal{S}_{k-1}$, sequences $s_j$, $\bar{y}_j$, $\rho_j$ with $j = k-p,\dots, k-2$, $u_0 = v_k$.
   \STATE Compute $g_{k-1} = \nabla l(x_{k-1}; \mathcal{S}_{k-1})$ and $\bar{g}_{k} = \nabla l(x_k; \mathcal{S}_{k-1})$.
   \STATE Set $s_{k-1} = x_k - x_{k-1}$ and $y_{k-1} = \bar{g}_k - g_{k-1}$.
   \STATE Compute $\Gamma_{k-1}$ by \eqref{eq_Gamma_k-1} and $c_k$ by \eqref{eq_H_k,0}
   \STATE Calculate $\bar{y}_{k-1}$ through \eqref{eq_yk-bar} and $\rho_{k-1} = (s_{k-1}^T \bar{y}_{k-1})^{-1}$.
   \FOR{$i=0,\dots,\min\{ p, k-1 \}-1$}
        \STATE Compute $\nu_i = \rho_{k-i-1}u_i^T s_{k-i-1}$.
        \STATE Compute $u_{i+1} = u_i - \nu_i \bar{y}_{k-i-1}$.
   \ENDFOR
   \STATE Calculate $\alpha_0 = c_k^{-1} u_p$.
   \FOR{$i=0,\dots,\min\{ p, k-1 \}-1$}
        \STATE Compute $\zeta_i = \rho_{k-p+i} \alpha_i^T \bar{y}_{k-p+i}$.
        \STATE Compute $\alpha_{i+1} = \alpha_i + (\nu_{p-i-1} - \zeta_i)s_{k-p+i}$.
   \ENDFOR
   \STATE {\bfseries Output:} $\alpha_p = H_k v_k$.
\end{algorithmic}
\end{algorithm}

Algorithm \ref{alg_AdaL-BFGS}, a stochastic adaptive L-BFGS method, boasts the following two desirable features: first, by borrowing the idea of the L-BFGS method, the computational cost is now reduced to $\mathcal{O}(pd)$; second, adaptivity in $w_{k-1}$ and $q_{k-1}$ will enable us to make the control of $\lambda_m$, $\lambda_M$ much easier.

Before we formally characterize $\lambda_m$ and $\lambda_M$, the following assumption is needed.
\begin{assumption}  \label{assump_pw-L0-L1-smooth}
    Given any sample $\xi$, the following condition is satisfied: for any $x,y \in \mathbb{R}^d$
    \begin{equation}
        \Vert \nabla l(x; \xi) - \nabla l(y; \xi) \Vert \le (\gamma_0 + \gamma_1 \Vert \nabla l(x; \xi) \Vert)\Vert x - y \Vert    \nonumber
    \end{equation}
    where $\gamma_0 >0$ and $\gamma_1 \ge 0$ are two constants.
\end{assumption}
Given Assumptions \ref{assump_pw-L0-L1-smooth} and \ref{assump_bnd-var}, we can easily deduce Assumption \ref{assump_L0L1-smooth} by noting that 
\begin{align}   \label{eq_gamma-L-connection}
    \mathbb{E}\Vert \nabla l(x;\xi) - \nabla l(y;\xi) \Vert^2 &\le \mathbb{E}(\gamma_0 + \gamma_1 \Vert \nabla l(x;\xi) \Vert)^2 \Vert x-y \Vert^2 \nonumber \\
    &\le (2\gamma_0^2 + 2 \gamma_1^2 \mathbb{E}\Vert \nabla l(x;\xi) \Vert^2)\Vert x-y \Vert^2 \nonumber \\
    &\le \left(2(\gamma_0^2 + \gamma_1^2 \sigma^2) + 2\gamma_1^2 \Vert \nabla F(x) \Vert^2 \right) \Vert x-y \Vert^2 \nonumber \\
    &\le \big(\sqrt{2(\gamma_0^2 + \gamma_1^2 \sigma^2)} + \gamma_1 \sqrt{2}\Vert \nabla F(x) \Vert\big)^2 \Vert x-y \Vert^2
\end{align}
with $L_0 =\sqrt{2(\gamma_0^2 + \gamma_1^2 \sigma^2)}, L_1=\gamma_1 \sqrt{2}$ and we make use of $\mathbb{E}\Vert \nabla l(x;\xi) \Vert^2 = \mathbb{E}\Vert \nabla l(x;\xi) - \nabla F(x) \Vert^2 + \Vert \nabla F(x) \Vert^2 \le \Vert \nabla F(x) \Vert^2 + \sigma^2$. 

\begin{remark}
    In \cite{wang2017stochastic,zhang2021faster}, a special case of this assumption is adopted to bound $\lambda_m$ and $\lambda_M$, where the authors set $\gamma_1$ in Assumption \ref{assump_pw-L0-L1-smooth} to be zero. However, due to the non-uniform  smoothness condition in our setting (see Assumption \ref{assump_L0L1-smooth}), $\gamma_1$ is in general nonzero.
\end{remark}

 Define 
    \begin{equation}    \label{eq_Gamma_k-1}
        \Gamma_{k-1} = \gamma_0( 1 + e^{\gamma_1/L_0} / L_0) + \frac{\gamma_1^2}{m_{k-1}}\sum_{l \in \mathcal{S}_{k-1}}\Vert \nabla l(x_{k-1};\xi_{l,k-1}) \Vert 
    \end{equation}
    with $L_0 = \sqrt{2(\gamma_0^2 + \gamma_1^2 \sigma^2)}$.
Then we are ready to state the results that specify the quantities of $\lambda_m$ and $\lambda_M$ in Assumption \ref{assump_pos-H}.

\begin{theorem} \label{thm_lambda_m}
    Suppose Assumption \ref{assump_pw-L0-L1-smooth} hold. For any $k \ge 0$, setting $q_k = q \Gamma_k^4$ such that $q_k \in (0,1)$ and $w_k = \kappa^2 / \Gamma_k^2$ for some $q, \kappa > 0$, we have
    \begin{equation}
        \lambda_m = \left( \delta + \frac{\kappa^2 p}{q \gamma_0^4}\left( \frac{1}{\delta} + \frac{\delta \gamma_0 + \kappa^2}{\gamma_0^3} + \frac{p + 2q \gamma_0^4}{2 p \gamma_0} \right) \right)^{-1}  \nonumber
    \end{equation}
    as that in Assumption \ref{assump_pos-H}.
\end{theorem}
\begin{theorem} \label{thm_lambda_M}
    Under Assumption \ref{assump_pw-L0-L1-smooth} with $q_k$, $w_k$ selecting the same as those in Theorem \ref{thm_lambda_m}, we could set $\lambda_M$ in Assumption \ref{assump_pos-H} as
    \begin{equation}
        \lambda_M = \frac{1}{\delta \gamma_0^2 \kappa^2 q} \frac{a^p - 1}{a - 1}  \nonumber
    \end{equation}
    where $a = 1 + \frac{2}{\delta \gamma_0 \kappa^2 q} + \big( \frac{1}{\delta \gamma_0 \kappa^2 q} \big)^2 .$
\end{theorem}

\begin{remark}
    If the loss function $l$ is $L_c$-Lipschitz uniformly continuous, i.e., for any $x,y$ and any sample $\xi$, $\Vert l(x;\xi) -  l(y;\xi)\Vert \le L_c \Vert x - y \Vert$, we could choose $q$ such that $0< q \le 0.8 (\gamma_0 + \gamma_1 L_c)^{-4}$.
\end{remark}

Then the following corollary provides some guidance on how to adjust $\lambda_m$, $\lambda_M$ as we expect.

\begin{corollary}   \label{coro_adjust-lambda}
    Suppose all conditions in Theorems \ref{thm_lambda_m} and \ref{thm_lambda_M} hold. Then $\lambda_m = \Omega\big( \frac{\delta q}{\kappa^2} + \frac{q}{\kappa^2 \delta} \big)$ and $\lambda_M = \mathcal{O}\left( \frac{1}{(\delta \kappa^{2} q)^{2p+1}} \right)$.
\end{corollary}

Corollary \ref{coro_adjust-lambda} essentially indicates that we could simply tune hyperparameters $\delta$ and $q$ to control $\lambda_m$ and $\lambda_M$ simultaneously. For instance, the ratio $\lambda_M / \lambda_m$ gets bigger if we decrease $\delta$ or $q$ and vice versa. Combining this observation together with Theorem \ref{thm_convergence}, we conclude that when $\beta$ is fixed, increasing $\delta$ or $q$ could decrease $\lambda_M / \lambda_m$ and hence correspondingly reduce the sample complexity. This in turn achieves convergence speedup.

If $\gamma_1 = 0$, Assumption \ref{assump_pw-L0-L1-smooth} reduces to the same one used in \cite{wang2017stochastic} to obtain bounds on $\lambda_m$ and $\lambda_M$. Even in this case, \cite{wang2017stochastic} falls short in controlling $\lambda_m$ and $\lambda_M$ by tuning hyperparameters. Therefore, our method is more general and can be adaptive to problem features.

%%%%%%%%%%%%%%%%%%%%%%%%%%%%%%%%%%%%%%%%%%%%%%%%%%%%%%%%%%%%%%%%%%%%%%%%%%%%%%%%%%%%%%%%%%%%%%%%%%%%%%%%%%%%%%%%%%%%%%%%%%%%%%%%%%%%%%%%%%%%%%%%%%%%%%%%%%%%%%%%%%%%%%%%%%%%%%%%%%%%%%%%%%%%
\section{Experiments}   \label{sec_exp}
In this section, we numerically evaluate the performances of our Clipped Stochastic Quasi-Newton method (i.e., Algorithm \ref{alg_clippSQN} and Algorithm \ref{alg_AdaL-BFGS}). We consider the practical \emph{empirical risk minimization} (ERM) version of problem \eqref{eq_online-obj}, i.e., \eqref{eq_finite-sum-obj}. The objective functions are constructed to maintain high degrees of non-convexity. Both first-order and quasi-Newton methods are compared with ours, including mini-batch SGD, Spider \cite{fang2018spider}, $(L_0,L_1)$-Spider \cite{reisizadeh2023variance}, SdLBFGS \cite{wang2017stochastic}.

We first generate a synthetic dataset as follows: the $i$-th data sample is denoted by $(a_i, b_i)$, where $a_i \in \mathbb{R}^d$ is the feature and $b_i \in \mathbb{R}$ is a scalar. Moreover, the feature vector $a_i$ is a sparse vector with $10\%$ nonzero elements and is drawn uniformly from $[0,1]^d$. Each $b_i$ is set as $b_i = \mathrm{sign}(u_i^T a_i)$ for some $u_i \in \mathbb{R}^d$ drawn uniformly from $[-1, 1]^d$. We set $d = 100$.

For mini-batch SGD and SdLBFGS, we choose a fixed batch size of 500. For Spider, $(L_0, L_1)$-Spider and our method, we set $|\mathcal{S}_1| = 2000$ and $|\mathcal{S}_2| = 100$. For two quasi-Newton methods, we select the memory size as $p = 5$ as suggested in \cite{nocedal1999numerical}. It is worth noting that given number of iteration $K$, our method exactly draw the same number of samples as those in Spider and $(L_0, L_1)$-Spider, when the batch sizes $|\mathcal{S}_1|$ and $\mathcal{S}_2$ are fixed, since Algorithm \ref{alg_AdaL-BFGS} involves no additional sampling.

\subsection{Non-convex robust linear regression}
First, we compare the above-mentioned algorithms under the following non-convex robust linear regression problem \cite{zhang2021faster}:
\begin{equation}
    F(x) = \frac{1}{n}\sum_{i=1}^n l(b_i - a_i^T x),    \nonumber
\end{equation}
where the non-convex loss function is defined by $l(x) = \log\big( \frac{x^2}{2} + 1 \big)$. The initial value $x_0$ is drawn from a multi-dimensional normal distribution. 

Figure \ref{fig_robustLR} shows the numerical performances of five algorithms given the robust linear regression problem. On one hand, we can observe that two quasi-Newton methods (SdLBFGS and ours) enjoy faster convergence speed, compared to first-order methods. On the other hand, our method finally achieves a smaller training error, although at the beginning converges slower compared to SdLBFGS. 
% This phenomenon is mainly due to that the stepsizes of our method suffer from the clipping effect. Specifically, when the gradient norm is large, the clipping operation restricts the stepsize selection to be more conservative.  

\begin{figure}[ht]
\vskip 0.2in
\begin{center}
\centerline{\includegraphics[width=7cm]{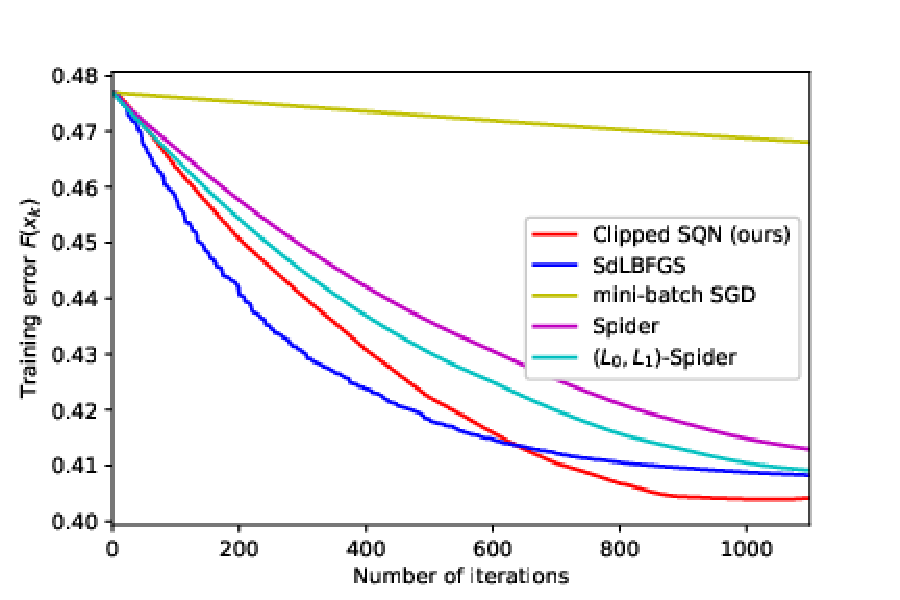}}
\caption{Training errors for algorithms solving non-convex robust linear regression problem}
\label{fig_robustLR}
\end{center}
\vskip -0.2in
\end{figure}

\subsection{Non-convex logistic regression}
We then test these algorithms by considering the following non-convex logistic regression problem \cite{chen2019stochastic}. In particular, the objective takes the following form:
\begin{equation}
    F(x) = -\frac{1}{n}\sum_{i=1}^n b_i \log(\sigma(a_i^T x)) + (1 - b_i) \log(\sigma(-a_i^T x))  \nonumber
\end{equation}
where $a_i$ denotes the $i$-th sample and $b_i$ is the corresponding label; $\sigma(x) = \frac{1}{1 + e^{-x}}$ is the sigmoid function. The starting point $x_0$ is drawn from $[-1, 1]^d$ uniformly. The numerical comparison among algorithms is illustrated in Figure \ref{fig_ncLogistic}. Clearly, our method outperforms other algorithms for achieving better speed and accuracy.

\begin{figure}[ht]
\vskip 0.2in
\begin{center}
\centerline{\includegraphics[width=7cm]{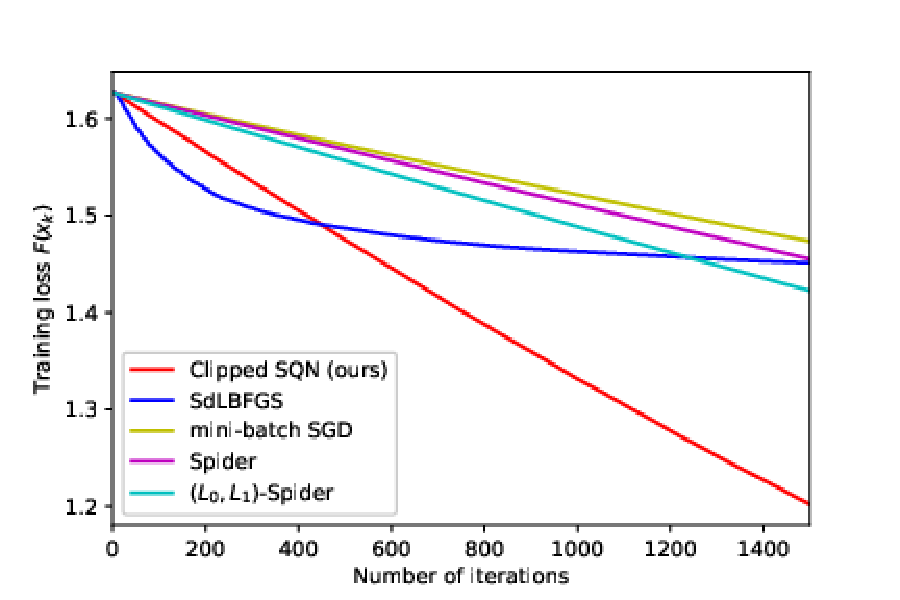}}
\caption{Training errors for algorithms solving non-convex logistic regression problem}
\label{fig_ncLogistic}
\end{center}
\vskip -0.2in
\end{figure}

%%%%%%%%%%%%%%%%%%%%%%%%%%%%%%%%%%%%%%%%%%%%%%%%%%%%%%%%%%%%%%%%%%%%%%%%%%%%%%%%%%%%%%%%%%%%%%%%%%%%%%%%%%%%%%%%%%%%%%%%%%%%%%%%%%%%%%%%%%%%%%%%%%%%%%%%%%%%%%%%%%%%%%%%%%%%%%%%%%%%%%%%%%%%
\section{Conclusion}
In this paper, we study the stochastic quasi-Newton method for non-convex optimization problems. Instead of limiting objective functions to be $L$-smooth, we focus on a more general non-uniform smoothness notion called $(L_0, L_1)$-smoothness. Then, we propose a stochastic quasi-Newton method by leveraging clipping techniques such that $\mathcal{O}(\epsilon^{-3})$ sample complexity can be achieved. Furthermore, we propose an adaptive L-BFGS method such that the eigenvalues of the Hessian inverse approximation can be controlled, hence allowing us to control the convergence speed. Numerical studies then are presented, which demonstrate a better performance of our method, compared to existing algorithms.

\bibliography{full_paper}

\begin{thebibliography}{10}

\bibitem{allen2018natasha}
Z.~Allen-Zhu.
\newblock Natasha 2: Faster non-convex optimization than sgd.
\newblock {\em Advances in neural information processing systems}, 31, 2018.

\bibitem{allen2016variance}
Z.~Allen-Zhu and E.~Hazan.
\newblock Variance reduction for faster non-convex optimization.
\newblock In {\em International conference on machine learning}, pages 699--707. PMLR, 2016.

\bibitem{arjevani2023lower}
Y.~Arjevani, Y.~Carmon, J.~C. Duchi, D.~J. Foster, N.~Srebro, and B.~Woodworth.
\newblock Lower bounds for non-convex stochastic optimization.
\newblock {\em Mathematical Programming}, 199(1-2):165--214, 2023.

\bibitem{boyd2004convex}
S.~P. Boyd and L.~Vandenberghe.
\newblock {\em Convex optimization}.
\newblock Cambridge university press, 2004.

\bibitem{byrd2016stochastic}
R.~H. Byrd, S.~L. Hansen, J.~Nocedal, and Y.~Singer.
\newblock A stochastic quasi-newton method for large-scale optimization.
\newblock {\em SIAM Journal on Optimization}, 26(2):1008--1031, 2016.

\bibitem{chen2019stochastic}
H.~Chen, H.-C. Wu, S.-C. Chan, and W.-H. Lam.
\newblock A stochastic quasi-newton method for large-scale nonconvex optimization with applications.
\newblock {\em IEEE transactions on neural networks and learning systems}, 31(11):4776--4790, 2019.

\bibitem{defazio2014saga}
A.~Defazio, F.~Bach, and S.~Lacoste-Julien.
\newblock Saga: A fast incremental gradient method with support for non-strongly convex composite objectives.
\newblock {\em Advances in neural information processing systems}, 27, 2014.

\bibitem{fang2018spider}
C.~Fang, C.~J. Li, Z.~Lin, and T.~Zhang.
\newblock Spider: Near-optimal non-convex optimization via stochastic path-integrated differential estimator.
\newblock {\em Advances in neural information processing systems}, 31, 2018.

\bibitem{gehring2017convolutional}
J.~Gehring, M.~Auli, D.~Grangier, D.~Yarats, and Y.~N. Dauphin.
\newblock Convolutional sequence to sequence learning.
\newblock In {\em International conference on machine learning}, pages 1243--1252. PMLR, 2017.

\bibitem{ghadimi2013stochastic}
S.~Ghadimi and G.~Lan.
\newblock Stochastic first-and zeroth-order methods for nonconvex stochastic programming.
\newblock {\em SIAM Journal on Optimization}, 23(4):2341--2368, 2013.

\bibitem{ghadimi2016accelerated}
S.~Ghadimi and G.~Lan.
\newblock Accelerated gradient methods for nonconvex nonlinear and stochastic programming.
\newblock {\em Mathematical Programming}, 156(1-2):59--99, 2016.

\bibitem{hillar2013most}
C.~J. Hillar and L.-H. Lim.
\newblock Most tensor problems are np-hard.
\newblock {\em Journal of the ACM (JACM)}, 60(6):1--39, 2013.

\bibitem{johnson2013accelerating}
R.~Johnson and T.~Zhang.
\newblock Accelerating stochastic gradient descent using predictive variance reduction.
\newblock {\em Advances in neural information processing systems}, 26, 2013.

\bibitem{kohler2017sub}
J.~M. Kohler and A.~Lucchi.
\newblock Sub-sampled cubic regularization for non-convex optimization.
\newblock In {\em International Conference on Machine Learning}, pages 1895--1904. PMLR, 2017.

\bibitem{lei2017non}
L.~Lei, C.~Ju, J.~Chen, and M.~I. Jordan.
\newblock Non-convex finite-sum optimization via scsg methods.
\newblock {\em Advances in Neural Information Processing Systems}, 30, 2017.

\bibitem{li2021page}
Z.~Li, H.~Bao, X.~Zhang, and P.~Richt{\'a}rik.
\newblock Page: A simple and optimal probabilistic gradient estimator for nonconvex optimization.
\newblock In {\em International conference on machine learning}, pages 6286--6295. PMLR, 2021.

\bibitem{li2021zerosarah}
Z.~Li, S.~Hanzely, and P.~Richt{\'a}rik.
\newblock Zerosarah: Efficient nonconvex finite-sum optimization with zero full gradient computation.
\newblock {\em arXiv preprint arXiv:2103.01447}, 2021.

\bibitem{liu1989limited}
D.~C. Liu and J.~Nocedal.
\newblock On the limited memory bfgs method for large scale optimization.
\newblock {\em Mathematical programming}, 45(1-3):503--528, 1989.

\bibitem{lucchi2015variance}
A.~Lucchi, B.~McWilliams, and T.~Hofmann.
\newblock A variance reduced stochastic newton method.
\newblock {\em arXiv preprint arXiv:1503.08316}, 2015.

\bibitem{mei2021leveraging}
J.~Mei, Y.~Gao, B.~Dai, C.~Szepesvari, and D.~Schuurmans.
\newblock Leveraging non-uniformity in first-order non-convex optimization.
\newblock In {\em International Conference on Machine Learning}, pages 7555--7564. PMLR, 2021.

\bibitem{merity2017regularizing}
S.~Merity, N.~S. Keskar, and R.~Socher.
\newblock Regularizing and optimizing lstm language models. arxiv 2017.
\newblock {\em arXiv preprint arXiv:1708.02182}, 2017.

\bibitem{mokhtari2014res}
A.~Mokhtari and A.~Ribeiro.
\newblock Res: Regularized stochastic bfgs algorithm.
\newblock {\em IEEE Transactions on Signal Processing}, 62(23):6089--6104, 2014.

\bibitem{moritz2016linearly}
P.~Moritz, R.~Nishihara, and M.~Jordan.
\newblock A linearly-convergent stochastic l-bfgs algorithm.
\newblock In {\em Artificial Intelligence and Statistics}, pages 249--258. PMLR, 2016.

\bibitem{nguyen2017sarah}
L.~M. Nguyen, J.~Liu, K.~Scheinberg, and M.~Tak{\'a}{\v{c}}.
\newblock Sarah: A novel method for machine learning problems using stochastic recursive gradient.
\newblock In {\em International conference on machine learning}, pages 2613--2621. PMLR, 2017.

\bibitem{nocedal1999numerical}
J.~Nocedal and S.~J. Wright.
\newblock {\em Numerical optimization}.
\newblock Springer, 1999.

\bibitem{peters2019knowledge}
M.~E. Peters, M.~Neumann, R.~L. Logan~IV, R.~Schwartz, V.~Joshi, S.~Singh, and N.~A. Smith.
\newblock Knowledge enhanced contextual word representations.
\newblock {\em arXiv preprint arXiv:1909.04164}, 2019.

\bibitem{pham2020proxsarah}
N.~H. Pham, L.~M. Nguyen, D.~T. Phan, and Q.~Tran-Dinh.
\newblock Proxsarah: An efficient algorithmic framework for stochastic composite nonconvex optimization.
\newblock {\em The Journal of Machine Learning Research}, 21(1):4455--4502, 2020.

\bibitem{qian2021understanding}
J.~Qian, Y.~Wu, B.~Zhuang, S.~Wang, and J.~Xiao.
\newblock Understanding gradient clipping in incremental gradient methods.
\newblock In {\em International Conference on Artificial Intelligence and Statistics}, pages 1504--1512. PMLR, 2021.

\bibitem{reddi2016stochastic}
S.~J. Reddi, A.~Hefny, S.~Sra, B.~Poczos, and A.~Smola.
\newblock Stochastic variance reduction for nonconvex optimization.
\newblock In {\em International conference on machine learning}, pages 314--323. PMLR, 2016.

\bibitem{reisizadeh2023variance}
A.~Reisizadeh, H.~Li, S.~Das, and A.~Jadbabaie.
\newblock Variance-reduced clipping for non-convex optimization.
\newblock {\em arXiv preprint arXiv:2303.00883}, 2023.

\bibitem{robbins1951stochastic}
H.~Robbins and S.~Monro.
\newblock A stochastic approximation method.
\newblock {\em The annals of mathematical statistics}, pages 400--407, 1951.

\bibitem{schraudolph2007stochastic}
N.~N. Schraudolph, J.~Yu, and S.~G{\"u}nter.
\newblock A stochastic quasi-newton method for online convex optimization.
\newblock In {\em Artificial intelligence and statistics}, pages 436--443. PMLR, 2007.

\bibitem{sohl2014fast}
J.~Sohl-Dickstein, B.~Poole, and S.~Ganguli.
\newblock Fast large-scale optimization by unifying stochastic gradient and quasi-newton methods.
\newblock In {\em International Conference on Machine Learning}, pages 604--612. PMLR, 2014.

\bibitem{sun2023convergence}
L.~Sun, A.~Karagulyan, and P.~Richtarik.
\newblock Convergence of stein variational gradient descent under a weaker smoothness condition.
\newblock In {\em International Conference on Artificial Intelligence and Statistics}, pages 3693--3717. PMLR, 2023.

\bibitem{wang2017stochastic}
X.~Wang, S.~Ma, D.~Goldfarb, and W.~Liu.
\newblock Stochastic quasi-newton methods for nonconvex stochastic optimization.
\newblock {\em SIAM Journal on Optimization}, 27(2):927--956, 2017.

\bibitem{yamashita2018convolutional}
R.~Yamashita, M.~Nishio, R.~K.~G. Do, and K.~Togashi.
\newblock Convolutional neural networks: an overview and application in radiology.
\newblock {\em Insights into imaging}, 9:611--629, 2018.

\bibitem{zhang2020improved}
B.~Zhang, J.~Jin, C.~Fang, and L.~Wang.
\newblock Improved analysis of clipping algorithms for non-convex optimization.
\newblock {\em Advances in Neural Information Processing Systems}, 33:15511--15521, 2020.

\bibitem{zhang2019gradient}
J.~Zhang, T.~He, S.~Sra, and A.~Jadbabaie.
\newblock Why gradient clipping accelerates training: A theoretical justification for adaptivity.
\newblock {\em arXiv preprint arXiv:1905.11881}, 2019.

\bibitem{zhang2021faster}
Q.~Zhang, F.~Huang, C.~Deng, and H.~Huang.
\newblock Faster stochastic quasi-newton methods.
\newblock {\em IEEE Transactions on Neural Networks and Learning Systems}, 33(9):4388--4397, 2021.

\bibitem{zhao2021convergence}
S.-Y. Zhao, Y.-P. Xie, and W.-J. Li.
\newblock On the convergence and improvement of stochastic normalized gradient descent.
\newblock {\em Science China Information Sciences}, 64:1--13, 2021.

\bibitem{zhou2018stochastic}
D.~Zhou, P.~Xu, and Q.~Gu.
\newblock Stochastic variance-reduced cubic regularized newton methods.
\newblock In {\em International Conference on Machine Learning}, pages 5990--5999. PMLR, 2018.

\bibitem{zhou2020stochastic}
D.~Zhou, P.~Xu, and Q.~Gu.
\newblock Stochastic nested variance reduction for nonconvex optimization.
\newblock {\em The Journal of Machine Learning Research}, 21(1):4130--4192, 2020.

\end{thebibliography}
\bibliographystyle{abbrv}

\appendix
\onecolumn
\section{Proof of Proposition \ref{prop_cross-entropy-smooth}}
    A straightforward calculation gives
    $$
        \nabla F(x) = \frac{y}{\hat{y}} \hat{y}(1 - \hat{y}) u = y (1 - \hat{y}) u
    $$
    $$
        \nabla^2_{xx} F(x) = - y \hat{y}(1 - \hat{y}) u u^T .
    $$
    Thus,
    $$
        \Vert \nabla^2 F(x) \Vert = \hat{y} \Vert u \Vert \Vert \nabla F(x) \Vert \le \Vert u \Vert \Vert \nabla F(x) \Vert
    $$
    by noting $\hat{y} \in (0, 1)$. This completes the proof.

\section{Proof of Theorem \ref{thm_convergence}}

We first present the descent lemma.
\begin{lemma}   \label{lmm_descent-lemma}
    Suppose Assumptions \ref{assump_L0L1-smooth}-\ref{assump_H-dependence} hold. Then for any $\beta > 0, c > 0$ and by selecting $\eta_k \le \frac{\beta c}{\Vert v_k \Vert}$, we have
    \begin{equation}
        F(x_{k+1}) \le F(x_k) - \eta_k \left(h_1^{\beta}(\lambda) - \frac{L_0 \lambda_M^2}{2} \eta_k \right) \Vert v_k \Vert^2 + h_2^{\beta}(\lambda)\eta_k \Vert v_k - \nabla F(x_k) \Vert^2   \nonumber
    \end{equation}
    where $h_1^{\beta}(c) = \lambda_m - \frac{\beta \lambda_M^2}{4}(2 + 3L_1 c)$ and $h_2^{\beta}(c) = \frac{1}{2\beta} + \frac{L_1 \lambda_M^2 \beta c}{4}$. Moreover, $h_1^{\beta}(c) > 0$ when $\beta < \frac{4 \lambda_m}{2 + L_1 c}$.
\end{lemma}

\begin{proof}
    By Assumption \ref{assump_L0L1-smooth}, we have that for any $x,y$
    \begin{align*}
        \Vert \nabla F(x) - \nabla F(y) \Vert^2 &= \Vert \mathbb{E}[\nabla l(x;\xi) - \nabla l(y;\xi)] \Vert^2   \\
        &\le \mathbb{E}\Vert \nabla l(x;\xi) - \nabla l(y;\xi) \Vert^2 \le (L_0 + L_1 \Vert \nabla F(x) \Vert)^2 \Vert x - y \Vert^2
    \end{align*}
    which indicates $F$ is $(L_0, L_1)$-Lipschitz smooth. Then, we have
    \begin{align*}
        F(x_{k+1}) &\le F(x_k) + \langle \nabla F(x_k), x_{k+1} - x_k \rangle + \frac{1}{2}(L_0 + L_1\Vert \nabla F(x_k) \Vert) \Vert x_{k+1} - x_k \Vert^2  \\
        &= F(x_k) - \eta_k \langle \nabla F(x_k), H_k v_k \rangle + \frac{L_0}{2}\eta_k^2 \Vert H_k v_k \Vert^2 + \frac{L_1}{2}\eta_k^2 \Vert \nabla F(x_k) \Vert \Vert H_k v_k \Vert^2     \\
        &= F(x_k) - \eta_k \langle \nabla F(x_k) - v_k, H_k v_k \rangle - \eta_k v_k^T H_k v_k + \frac{1}{2}(L_0 + L_1 \Vert \nabla F(x_k) \Vert) \eta_k^2 \Vert H_k v_k \Vert^2    \\
        &\overset{(a)}{\le} F(x_k) - \eta_k \langle \nabla F(x_k) - v_k, H_k v_k \rangle - \eta_k \lambda_m \Vert v_k \Vert^2 + \frac{1}{2}(L_0 + L_1 \Vert \nabla F(x_k) \Vert) \eta_k^2 \lambda_M^2 \Vert v_k \Vert^2     \\
        &\overset{(b)}{\le} F(x_k) + \frac{\eta_k}{2\beta}\Vert v_k - \nabla F(x_k) \Vert^2 + \frac{\beta \eta_k}{2} \lambda_M^2 \Vert v_k \Vert^2 - \eta_k \lambda_m \Vert v_k \Vert^2 + \frac{1}{2}(L_0 + L_1 \Vert \nabla F(x_k) \Vert) \eta_k^2 \lambda_M^2 \Vert v_k \Vert^2   \\
        &= F(x_k) - \eta_k \big(\lambda_m - \frac{\beta}{2}\lambda_M^2 - \frac{L_0 \lambda_M^2}{2} \eta_k \big) \Vert v_k \Vert^2 + \frac{\eta_k}{2\beta} \Vert v_k - \nabla F(x_k) \Vert^2 + \underbrace{\frac{L_1 \lambda_M^2}{2}\eta_k^2 \Vert \nabla F(x_k) \Vert \Vert v_k \Vert^2}_{\mathcal{T}} ,
    \end{align*}    
    where $(a)$ follows Assumption \ref{assump_pos-H}; $(b)$ follows the Young's inequality, i.e., for any vectors $v_1, v_2$, $\langle v_1, v_2 \rangle \le \frac{\beta}{2}\Vert v_1 \Vert^2 + \frac{1}{2\beta}\Vert v_2 \Vert^2$ with constant $\beta > 0$. In the following, we derive an upper bound for the term $\mathcal{T}$. Note that
    \begin{align*}
        \mathcal{T} &= \frac{L_1 \lambda_M^2}{2}\eta_k^2 \Vert v_k \Vert^2 \Vert \nabla F(x_k) - v_k + v_k \Vert   \\
        &\le \frac{L_1 \lambda_M^2}{2}\eta_k^2 \Vert v_k \Vert^3 + \frac{L_1 \lambda_M^2}{2}\eta_k^2 \Vert v_k \Vert^2 \Vert v_k - \nabla F(x_k) \Vert  \\
        &\overset{(c)}{\le} \frac{3 L_1 \lambda_M^2}{4} \eta_k^2 \Vert v_k \Vert^3 + \frac{L_1 \lambda_M^2}{4}\eta_k^2 \Vert v_k \Vert \Vert v_k - \nabla F(x_k) \Vert^2  \\
        &\overset{(d)}{\le} \frac{3 L_1 \beta \lambda_M^2 c}{4} \eta_k \Vert v_k \Vert^2 + \frac{L_1 \beta \lambda_M^2 c}{4} \eta_k \Vert v_k - \nabla F(x_k) \Vert^2 
    \end{align*}
    where $(c)$ follows the Young's inequality and $(d)$ is by $\eta_k \Vert v_k \Vert \le \beta c$. Combining these we have
    \begin{align*}
        F(x_{k+1}) \le F(x_k) - \eta_k \left( \lambda_m - \frac{\beta \lambda_M^2}{4}(2 + 3L_1 c) - \frac{L_0 \lambda_M^2}{2}\eta_k \right) \Vert v_k \Vert^2 + \eta_k \big( \frac{1}{2\beta} + \frac{L_1 \lambda_M^2 \beta c}{4} \big) \Vert v_k - \nabla F(x_k \Vert^2 
    \end{align*}
    which completes the proof.
\end{proof}

Then, the following lemma indicates that $v_k$ is essentially a good estimate of $\nabla F(x_k)$ when $|\mathcal{S}_1|$, $|\mathcal{S}_2|$, $r$ and stepsizes $\eta_k$ are chosen properly.
\begin{lemma}   \label{lmm_bound-vk}
    Suppose Assumptions \ref{assump_L0L1-smooth} and \ref{assump_bnd-var} hold. By choosing $|\mathcal{S}_1| = \frac{2\sigma^2}{\epsilon^2}$, $|\mathcal{S}_2| = \frac{4 h_1^{\beta}(c)^2}{\epsilon}$, $r = \frac{1}{\epsilon}$ and selecting $\eta_k \le \min \{ \frac{ h_1^{\beta}(c) \epsilon}{L_0 \lambda_M^2 \Vert v_k \Vert}, \frac{h_1^{\beta}(c) \epsilon}{L_1 \lambda_M^2 \Vert v_k \Vert^2} \}$ with $h_1^{\beta}(c)$ defined in Lemma \ref{lmm_descent-lemma}, we have for any $k\ge 0$ 
    \begin{equation}
        \mathbb{E}\Vert v_k - \nabla F(x_k) \Vert^2 \le \left( \frac{1}{2}e^{(L_1 / L_0)^2} +  \frac{3}{2} \frac{e^{(L_1 / L_0)^2}}{(L_1 / L_0)^2}\right) \epsilon^2 =: \tilde{m}(L_1/L_0) \epsilon^2.  \nonumber
    \end{equation}
    where $\tilde{m}(L_1 / L_0) := \left( \frac{1}{2}e^{(L_1 / L_0)^2} +  \frac{3}{2} \frac{e^{(L_1 / L_0)^2}}{(L_1 / L_0)^2}\right)$. Furthermore, when $L_1 = 0$ in Assumption \ref{assump_L0L1-smooth}, we have
    \begin{equation}
        \mathbb{E}\Vert v_k - \nabla F(x_k) \Vert^2 \le 2 \epsilon^2. \nonumber
    \end{equation}
\end{lemma}
\begin{proof}
    When $k ~\mathrm{mod}~ r = 0$, we have
    \begin{align*}
        \mathbb{E}\Vert v_k - \nabla F(x_k) \Vert^2 &= \mathbb{E}\Vert \nabla l(x_k; \mathcal{S}_1) - \nabla F(x_k) \Vert^2  \\
        &= \mathbb{E}\Vert \frac{1}{|\mathcal{S}_1|}\sum_{\xi_i \in \mathcal{S}_1} \nabla l(x_k; \xi_i) - \nabla F(x_k) \Vert^2  \\
        &\le \frac{1}{|\mathcal{S}_1|^2} \sum_{\xi_i \in \mathcal{S}_1} \mathbb{E}\Vert \nabla l(x_k; \xi_i) - \nabla F(x_k) \Vert^2  \\
        &\le \frac{\sigma^2}{|\mathcal{S}_1|} = \frac{\epsilon^2}{2} 
    \end{align*}
    by choosing $|\mathcal{S}|_1 = \frac{2 \sigma^2}{\epsilon^2}$.

    When $k ~ \mathrm{mod} ~ r \ne 0$, we have
    \begin{align*}
        \mathbb{E}\Vert v_{k+1} - \nabla F(x_{k+1}) \Vert^2 &= \mathbb{E}\Vert v_k + \nabla l(x_{k+1}; \mathcal{S}_2) - \nabla l(x_k; \mathcal{S}_2) - \nabla F(x_{k+1}) \Vert^2  \\
        &= \mathbb{E}\Vert v_k - \nabla F(x_k) \Vert^2 + \mathbb{E}\Vert \nabla l(x_{k+1}; \mathcal{S}_2) - \nabla F(x_{k+1}) + \nabla F(x_k) - \nabla l(x_k; \mathcal{S}_2) \Vert^2  \\
        &\le \mathbb{E}\Vert v_k - \nabla F(x_k) \Vert^2 + \frac{1}{|\mathcal{S}_2|} \mathbb{E}\Vert \nabla l(x_{k+1};\xi) - \nabla l(x_k;\xi) \Vert^2   \\
        &\le \mathbb{E}\Vert v_k - \nabla F(x_k) \Vert^2 + \frac{1}{|\mathcal{S}_2|}\mathbb{E}\left[(L_0 + L_1 \Vert \nabla F(x_k) \Vert)^2\Vert x_{k+1} - x_k \Vert^2 \right] .
    \end{align*}
    Noting that
    \begin{align*}
        \frac{1}{|\mathcal{S}_2|}(L_0 + L_1 \Vert \nabla F(x_k))^2\Vert x_{k+1} - x_k \Vert^2 &= \frac{1}{|\mathcal{S}_2|}(L_0 + L_1 \Vert \nabla F(x_k) \Vert)^2\eta_k^2 \Vert H_k v_k \Vert^2   \\
        &\le \frac{1}{|\mathcal{S}_2|}(L_0 + L_1 \Vert \nabla F(x_k) \Vert)^2 \eta_k^2 \lambda_M^2 \Vert v_k \Vert^2    \\
        &\le \frac{2 L_0^2 \lambda_M^2}{|\mathcal{S}_2|} \eta_k^2 \Vert v_k \Vert^2 + \frac{2 L_1^2 \lambda_M^2}{|\mathcal{S}_2|} \eta_k^2 \Vert \nabla F(x_k) \Vert^2 \Vert v_k \Vert^2    \\
        &\le \frac{2 L_0^2 \lambda_M^2}{|\mathcal{S}_2|} \eta_k^2 \Vert v_k \Vert^2 + \frac{4 L_1^2 \lambda_M^2}{|\mathcal{S}_2|} \eta_k^2 \Vert v_k \Vert^4 + \frac{4 L_1^2 \lambda_M^2}{|\mathcal{S}_2|} \eta_k^2 \Vert v_k \Vert^2 \Vert v_k - \nabla F(x_k) \Vert^2  \\
        &\le \frac{6 h_1^{\beta}(c)^2}{|\mathcal{S}_2|} \epsilon^2 + \frac{4 h_1^{\beta}(c)^2}{|\mathcal{S}_2|} \left( \frac{L_1}{L_0} \right)^2 \Vert v_k - \nabla F(x_k) \Vert^2,
    \end{align*}
    we conclude that
    \begin{equation}
        \mathbb{E}\Vert v_{k+1} - \nabla F(x_{k+1}) \Vert^2 \le \left(1 + \frac{4 h_1^{\beta}(c)^2}{|\mathcal{S}_2|} \left( \frac{L_1}{L_0} \right)^2 \right) \mathbb{E}\Vert v_k - \nabla F(x_k) \Vert^2 + \frac{6 h_1^{\beta}(c)^2}{|\mathcal{S}_2|} \epsilon^2 .  \nonumber
    \end{equation}
    
    Consider the following sequence $x_{k+1} \le a x_k + b$, where $x_k$, $a \ge 1$, $b$ are all positive scalars. By induction, we obtain for any $0 \le k \le r$
    \begin{align*}
        x_k \le a^k x_0 + \sum_{l=0}^{k-1} a^l b \le a^q x_0 + \sum_{l=0}^{q - 1} a^l b = a^q x_0 + b \frac{a^q - 1}{a-1}
    \end{align*}
    and particularly in our case $a = 1 + \frac{4h_1^{\beta}(c)^2}{|\mathcal{S}_2|} \left( \frac{L_1}{L_0} \right)^2$, $b = \frac{6 h_1^{\beta}(c)^2}{|\mathcal{S}_2|} \epsilon^2$, $x_k = \mathbb{E}\Vert v_k - \nabla F(x_k) \Vert^2$, $x_0 = \frac{\epsilon^2}{2}$. 
    Setting $|\mathcal{S}_2| = \frac{4 h_1^{\beta}(c)^2}{\epsilon}$ and $r = \frac{1}{\epsilon}$, we have
    \begin{align*}
        a^q = \left( 1 + \frac{4 h_1^{\beta}(c)^2}{|\mathcal{S}_2|} \left( \frac{L_1}{L_0} \right)^2 \right)^{q} = \left( 1 + \left( \frac{L_1}{L_0} \right)^2 \epsilon \right)^{1/\epsilon} \le e^{(L_1 / L_0)^2}
    \end{align*}
    where we use the fact that given any positive $c_1, c_2$, $(1 + c_1 x)^{c_2 /x} \le e^{c_1 c_2}, \forall x > 0$. Moreover, note that
    \begin{align*}
        b \frac{a^q - 1}{a - 1} \le e^{(L_1 / L_0)^2} \epsilon^2 \frac{6 h_1^{\beta}(c)^2}{|\mathcal{S}_2|} \cdot \frac{|\mathcal{S}_2|}{4 h_1^{\beta}(c)^2} \left( \frac{L_0}{L_1} \right)^2 = \frac{3}{2} \epsilon^2 \frac{e^{(L_1 / L_0)^2}}{(L_1 / L_0)^2}.
    \end{align*}
    Combining this with the case when $k ~ \mathrm{mod} ~ q = 0$, we finally have
    \begin{align*}
        \mathbb{E}\Vert v_k - \nabla F(x_k) \Vert^2 \le \left( \frac{1}{2}e^{(L_1 / L_0)^2} +  \frac{3}{2} \frac{e^{(L_1 / L_0)^2}}{(L_1 / L_0)^2}\right) \epsilon^2.
    \end{align*}
    To obtain the bound for $L_1 = 0$, we note that 
    $$
        \lim_{L_1 \to 0}\frac{e^{(L_1 / L_0)^2}}{(L_1 / L_0)^2} = 1
    $$
    which completes the proof.
\end{proof}

Then, we are ready to prove Theorem \ref{thm_convergence}.

\emph{Proof of Theorem \ref{thm_convergence}:}
Note that by Lemma \ref{lmm_descent-lemma}, we have
    \begin{equation}
        F(x_{k+1}) \le F(x_k) - \eta_k \left(h_1^{\beta}(c) - \frac{L_0 \lambda_M^2}{2} \eta_k \right) \Vert v_k \Vert^2 + h_2^{\beta}(c)\eta_k \Vert v_k - \nabla F(x_k) \Vert^2   \nonumber
    \end{equation}
where $h_1^{\beta}(c) = \lambda_m - \frac{\beta \lambda_M^2}{4}(2 + 3L_1 c)$ and $h_2^{\beta}(c) = \frac{1}{2 \beta} + \frac{L_1 \lambda_M^2 \beta c}{4}$. Moreover, when $\epsilon$ is small enough, selecting $\eta_k = \min \{\frac{h_1^{\beta}(c)}{2 L_0 \lambda_M^2}, \frac{h_1^{\beta}(c) \epsilon}{L_0 \lambda_M^2 \Vert v_k \Vert},  \frac{h_1^{\beta}(c) \epsilon}{L_1 \lambda_M^2 \Vert v_k \Vert^2}\}$ yields
\begin{align*}
    \eta_k \left(h_1^{\beta}(c) - \frac{L_0 \lambda_M^2}{2} \eta_k \right) \Vert v_k \Vert^2 &\ge \frac{h_1^{\beta}(c)}{2} \eta_k \Vert v_k \Vert^2   \\
    &\ge \frac{h_1^{\beta}(c)^2 \epsilon^2}{2 L_0 \lambda_M^2} \min \left\{ \frac{1}{2} \left\Vert \frac{v_k}{\epsilon} \right\Vert^2, \left\Vert \frac{v_k}{\epsilon} \right\Vert \right\}   \\
    &\ge \frac{h_1^{\beta}(c)^2}{2 L_0 \lambda_M^2} \epsilon \Vert v_k \Vert - \frac{h_1^{\beta}(c)^2}{L_0 \lambda_M^2} \epsilon^2,
\end{align*}
where we use the fact $\min \{ x^2 / 2, |x| \} \ge |x| - 2, \forall x \in \mathbb{R}$. Thus, by taking expectation we obtain
\begin{align*}
    \mathbb{E}[F(x_{k+1})] \le \mathbb{E}[F(x_k)] - \frac{h_1^{\beta}(c)^2}{2 L_0 \lambda_M^2} \epsilon \mathbb{E}\Vert v_k \Vert + \frac{h_1^{\beta}(c)^2}{L_0 \lambda_M^2} \epsilon^2 + \frac{h_2^{\beta}(c) h_1^{\beta}(c)}{2 L_0 \lambda_M^2} \mathbb{E}\Vert v_k - \nabla F(x_k) \Vert^2 .
\end{align*}
After applying Lemma \ref{lmm_bound-vk} and rearranging the terms, we obtain
\begin{align*}
    \frac{h_1^{\beta}(c)^2}{2 L_0 \lambda_M^2} \epsilon \mathbb{E}\Vert v_k \Vert \le \mathbb{E}[F(x_k)] - \mathbb{E}[F(x_{k+1})] + \frac{h_1^{\beta}(c)^2}{L_0 \lambda_M^2} \epsilon^2 +\frac{h_2^{\beta}(c) h_1^{\beta}(c)}{2 L_0 \lambda_M^2} \tilde{m}(L_1 / L_0) \epsilon^2 .
\end{align*}
Summing over $k$ and noting $\mathbb{E}[F(x_0) - F(x_K)] \le \mathbb{E}[F(x_0) - F(x^*)] =: \Delta_0$ further indicate
\begin{align*}
    \frac{1}{K}\sum_{k=0}^{K-1} \mathbb{E}\Vert v_k \Vert &\le \frac{2 L_0 \lambda_M^2 \Delta_0}{\big( h_1^{\beta}(c) \big)^2 \epsilon K} + 2 \epsilon + \frac{h_2^{\beta}(c)}{h_1^{\beta}(c)} \tilde{m}(L_1/L_0) \epsilon  \\
    &\le (3 + \tilde{m}(L_1 / L_0) \beta^{-2}) \epsilon  
\end{align*}
when we choose $\beta \le \frac{\lambda_m}{1 + \lambda_M^2}$, $c \le \frac{4\lambda_m - 2\beta(1 + \lambda_M^2)}{L_1 \lambda_M^2 \beta (3 + \beta^2)}$ and $K = \lceil \frac{2 L_0 \lambda_M^2 \Delta_0}{h_1^{\beta}(c)^2 \epsilon^2} \rceil$. Then, note that
\begin{align*}
    \mathbb{E}\Vert \nabla F(\tilde{x}_K) \Vert &= \frac{1}{K}\sum_{k=0}^{K-1} \mathbb{E}\Vert \nabla F(x_k) \Vert  \\
    &\le \frac{1}{K}\sum_{k=0}^{K-1} \mathbb{E}\Vert v_k \Vert + \mathbb{E}\Vert v_k - \nabla F(x_k) \Vert   \\
    &\le (3 + \tilde{m}(L_1 / L_0)(1 + \beta^{-2})) \epsilon
\end{align*}
where $\tilde{m}(L_1 / L_0) = \frac{1}{2}e^{(L_1 / L_0)^2} +  \frac{3}{2} \frac{e^{(L_1 / L_0)^2}}{(L_1 / L_0)^2}$ and note $\tilde{m}(x) \ge 1, \forall x > 0$, which indicates $\sqrt{\tilde{m}(L_1 / L_0)} \le \tilde{m}(L_1 / L_0)$. 

In terms of the sample complexity, we note that the total number of samples we use during Algorithm \ref{alg_clippSQN} is 
\begin{align*}
    \left\lceil K \cdot \frac{1}{r} \right\rceil \vert S_1 \vert + K\vert S_2 \vert \le \frac{4 \sigma^2 L_0 \lambda_M^2 \Delta_0}{h_1^{\beta}(c)^2 \epsilon^3} + \frac{8 L_0 \lambda_M^2 \Delta_0}{\epsilon^3} = \mathcal{O}\left(\frac{\lambda_M^2}{\lambda_m^2}\epsilon^{-3}\right) 
\end{align*}
which completes the proof.

\section{Proofs of Lemmas \ref{lmm_pd-Hk} and \ref{lmm_H_k,i-pd}}
\emph{Proof of Lemma \ref{lmm_pd-Hk}:}
Note that from \eqref{eq_yk-hat}
    \begin{align*}
        s_{k-1}^T \hat{y}_{k-1} &= \hat{w}_{k-1} \hat{\theta}_{k-1} (s_{k-1}^T y_{k-1} - s_{k-1}^T B_{k-1} s_{k-1}) + \hat{w}_{k-1} s_{k-1}^T B_{k-1} s_{k-1}   \\
        &= \left\{
        \begin{array}{ccc}
            \hat{w}_{k-1} \hat{q}_{k-1} s_{k-1}^T B_{k-1} s_{k-1} & , & s_{k-1}^T y_{k-1} < \hat{q}_{k-1} s_{k-1}^T B_{k-1} s_{k-1}  \\
            \hat{w}_{k-1} s_{k-1}^T y_{k-1} & , & \mathrm{otherwise}
        \end{array}
        \right .    \\
        &\ge \hat{w}_{k-1} \hat{q}_{k-1} s_{k-1}^T B_{k-1} s_{k-1} .
    \end{align*}
    Thus if $B_{k-1}$ is positive definite, $\hat{\rho}_{k-1} > 0$. Then for any $z$ we have
    \begin{equation*}
        z^T H_k z = \hat{\rho}_{k-1} (s_{k-1}^T z)^2 + \left[ (I - \hat{\rho}_{k-1} \hat{y}_{k-1} s_{k-1}^T) z \right]^T H_{k-1} \left[ (I - \hat{\rho}_{k-1} \hat{y}_{k-1} s_{k-1}^T) z \right] > 0
    \end{equation*}
    which indicates $H_k$ is positive definite and hence $B_k$ is positive definite.

\emph{Proof of Lemma \ref{lmm_H_k,i-pd}:}
Suppose $s_j^T \bar{y}_j \ge w_j q_j s_j^T H_{j+1,0}^{-1} s_{j}$ for $j=k-p, \dots, k-2$ and $H_{k, i-1} \succ 0$. Then similar to the proof of Lemma \ref{lmm_pd-Hk}, it is straightforward to show that
    $$
        s_{k-1}^T \bar{y}_{k-1} \ge w_{k-1} q_{k-1} s_{k-1}^T H_{k,0}^{-1} s_{k-1} > 0.
    $$
    Therefore, noting that $H_{k,0} \succ 0$ by \eqref{eq_H_k,0}, we have $H_{k,i} \succ 0, \forall i=1,\dots,p$, similar to the proof of Lemma \ref{lmm_pd-Hk}.

\section{Proofs of Theorems \ref{thm_lambda_m} and \ref{thm_lambda_M}}
We first present some useful lemmas.

\begin{lemma}   \label{lmm_bnd-grad-onpath}
    Suppose $l(\cdot;\xi)$ is $(\gamma_0, \gamma_1)$-smooth for every $\xi$. Let $c>0$ be a constant. Then for any $y$ such that $\Vert y -x \Vert \le c$, we have $\Vert \nabla l(\tilde{x}(t); \xi) \Vert \le e^{c\gamma_1} \big(c \gamma_0 + \gamma_1 \Vert \nabla l(x;\xi) \Vert \big)$, where $\tilde{x}(t) := x + t(y-x), t \in [0,1]$.
\end{lemma}
\begin{proof}
    It follows the proof of Lemma A.2 in \cite{zhang2020improved}.
\end{proof}

\begin{lemma}   \label{lmm_useful-facts}
    Define 
    $$
        \Gamma_{k-1} = \gamma_0\big( 1 + \frac{e^{\gamma_1/L_0}}{L_0} \big) + \frac{\gamma_1^2}{m_{k-1}}\sum_{l \in \mathcal{S}_{k-1}}\Vert \nabla l(x_{k-1};\xi_{l,k-1}) \Vert .
    $$
    The following conditions hold:
    \begin{equation}
        \frac{y_{k-1}^T y_{k-1}}{s_{k-1}^T y_{k-1}} \le \Gamma_{k-1}, ~ \frac{y_{k-1}^T y_{k-1}}{s_{k-1}^T s_{k-1}} \le \Gamma_{k-1}^2, ~ \frac{s_{k-1}^T y_{k-1}}{s_{k-1}^T s_{k-1}} \le \Gamma_{k-1} .
    \end{equation}
\end{lemma}
\begin{proof}
    By the definition of $y_{k-1}$ and letting $m_{k-1} = |\mathcal{S}_{k-1}|$ we have
    \begin{align*}
        y_{k-1} &= \frac{1}{m_{k-1}} \sum_{l \in \mathcal{S}_{k-1}} \nabla l(x_k; \xi_{l, k-1}) - \nabla l(x_{k-1}; \xi_{l, k-1})  \\
        &= \frac{1}{m_{k-1}} \sum_{l \in \mathcal{S}_{k-1}} \int_{0}^{1} \nabla^2_{xx} l(x_{k-1} + t(x_k - x_{k-1}); \xi_{l, k-1}) (x_k - x_{k-1}) dt   \\
        &=: G_{k-1} (x_k - x_{k-1})
    \end{align*}
    where
    $$
        G_{k-1} = \frac{1}{m_{k-1}} \sum_{l \in \mathcal{S}_{k-1}} \int_{0}^{1} \nabla^2_{xx} l(x_{k-1} + t(x_k - x_{k-1}); \xi_{l, k-1}) dt .
    $$
    Next, we derive the bound for $\Vert G_{k-1} \Vert$. Note that when $\Vert x_k - x_{k-1} \Vert \le c$ by Lemma \ref{lmm_bnd-grad-onpath} we have
    \begin{align*}
        \Vert G_{k-1} \Vert &\le \frac{1}{m_{k-1}} \sum_{l \in \mathcal{S}_{k-1}} \int_{0}^{1} \Vert \nabla^2_{xx} l(x_{k-1} + t(x_k - x_{k-1}); \xi_{l, k-1}) \Vert dt  \\
        &\le \gamma_0 + \frac{\gamma_1}{m_{k-1}} \sum_{l \in \mathcal{S}_{k-1}} \int_{0}^{1} \Vert \nabla l(x_{k-1} + t(x_k - x_{k-1}); \xi_{l, k-1})\Vert dt  \\
        &\le \gamma_0 + \frac{\gamma_1}{m_{k-1}}\sum_{l \in \mathcal{S}_{k-1}} e^{c\gamma_1} \big(c \gamma_0 + \gamma_1 \Vert \nabla l(x_{k-1};\xi_{l,k-1}) \Vert \big) \\
        &= \gamma_0(1 + c e^{c \gamma_1}) + \frac{\gamma_1^2}{m_{k-1}} \sum_{l \in \mathcal{S}_{k-1}} \Vert \nabla l(x_{k-1};\xi_{l,k-1}) \Vert 
    \end{align*}
    and the selection of $\eta_{k-1}$ guarantees that $\Vert x_k - x_{k-1} \Vert \le \eta_{k-1} \Vert v_{k-1} \Vert \le \frac{1}{L_0}$, implying $c = L_0^{-1}$.
    Further noting that
    \begin{align*}
        \frac{y_{k-1}^T y_{k-1}}{s_{k-1}^T y_{k-1}} &= \frac{s_{k-1}^T G^2_{k-1} s_{k-1}}{s_{k-1}^T G_{k-1} s_{k-1}} = \frac{z_{k-1}^T G_{k-1} z_{k-1}}{z_{k-1}^T z_{k-1}} \le \Vert G_{k-1} \Vert \le \Gamma_{k-1} , \\
        \frac{y_{k-1}^T y_{k-1}}{s_{k-1}^T s_{k-1}} &= \frac{s_{k-1}^T G_{k-1}^2 s_{k-1}}{s_{k-1}^T s_{k-1}} \le \Vert G_{k-1} \Vert^2 \le \Gamma_{k-1}^2 ,  \\
        \frac{s_{k-1}^T y_{k-1}}{s_{k-1}^T s_{k-1}} &= \frac{s_{k-1}^T G_{k-1} s_{k-1}}{s_{k-1}^T s_{k-1}} \le \Vert G_{k-1} \Vert \le \Gamma_{k-1} ,
    \end{align*}
    where $z_{k-1} := G_{k-1}^{\frac{1}{2}} s_{k-1}$, we complete the proof.
\end{proof}

We formally provide the proofs for Theorems \ref{thm_lambda_m} and \ref{thm_lambda_M} in the following.

\emph{Proof of Theorem \ref{thm_lambda_m}:}
We alternatively consider the update of $B_{k,i} = H_{k,i}^{-1}$, which is 
    $$
        B_{k,i} = B_{k, i-1} - \frac{B_{k,i-1}s_j s_j^T B_{k,i-1}}{s_j^T B_{k,i-1}s_j} + \frac{\bar{y}_j \bar{y}_j^T}{s_j^T \bar{y}_j} .
    $$
    From Lemma \ref{lmm_H_k,i-pd} we obtain $B_{k,i-1} \succ 0$ and thus
    \begin{align*}
        \Vert B_{k,i} \Vert &\le \left\Vert B_{k, i-1} - \frac{B_{k,i-1}s_j s_j^T B_{k,i-1}}{s_j^T B_{k,i-1}s_j} \right\Vert + \frac{\bar{y}_j^T \bar{y}_j}{s_j^T \bar{y}_j}  \\
        &\le \Vert B_{k,i-1} \Vert + \frac{\bar{y}_j^T \bar{y}_j}{s_j^T \bar{y}_j} .
    \end{align*}
    Next we turn to bound $\frac{\bar{y}_j^T \bar{y}_j}{s_j^T \bar{y}_j}$. Noting that $s_j^T \bar{y}_j \ge w_j q_j s_j^T B_{j+1,0} s_j$, 
    \begin{align*}
        \frac{\bar{y}_j^T \bar{y}_j}{s_j^T \bar{y}_j} &\le \frac{w_j \Vert \theta_j y_j + (1-\theta_j)B_{j+1,0} s_j \Vert^2}{q_j s_j^T B_{j+1,0} s_j}  \\
        &= \frac{w_j \theta_j^2}{q_j c_{j+1}} \frac{y_j^T y_j}{s_j^T s_j} + \frac{w_j c_{j+1} (1 - \theta_j)^2}{q_j} + \frac{2 w_j \theta_j (1 - \theta_j)}{q_j} \frac{s_j^T y_j}{s_j^T s_j}  \\
        &\le \frac{w_j \Gamma^2_{j}}{q_j c_{j+1}} + \frac{w_j c_{j+1}}{q_j} + \frac{w_j \Gamma_j}{2 q_j} .
    \end{align*}
    Selecting $w_j = \frac{\kappa^2}{\Gamma_j^2}$ and noting $\Gamma_j \ge \gamma_0$, we obtain
    \begin{equation}
        \frac{\bar{y}_j^T \bar{y}_j}{s_j^T \bar{y}_j} \le \frac{\kappa^2}{q_j c_{j+1}} + \frac{\kappa^2 c_{j+1}}{q_j \gamma_0^2} + \frac{\kappa^2}{2 q_j \gamma_0} = \frac{\kappa^2}{q_j} \left( \frac{1}{\delta} + \frac{\delta \gamma_0 + \kappa^2}{\gamma_0^3} + \frac{1}{2\gamma_0} \right)
    \end{equation}
    where we note that $\delta \le c_j \le \delta + \kappa^2 / \Gamma_j \le \delta + \kappa^2 / \gamma_0$. Therefore, by setting $q_j = q \Gamma_j^4$ with constant $q$ for any $j$, we conclude that
    $$
        \Vert B_{k,p} \Vert \le \Vert B_{k,0} \Vert + \frac{\kappa^2 p}{q \gamma_0^4}\left( \frac{1}{\delta} + \frac{\delta \gamma_0 + \kappa^2}{\gamma_0^3} + \frac{1}{2\gamma_0} \right) \le \delta + \frac{\kappa^2 p}{q \gamma_0^4}\left( \frac{1}{\delta} + \frac{\delta \gamma_0 + \kappa^2}{\gamma_0^3} + \frac{p + 2q \gamma_0^4}{2 p \gamma_0} \right)
    $$
    which further implies that
    $$
        \lambda_m = \left( \delta + \frac{\kappa^2 p}{q \gamma_0^4}\left( \frac{1}{\delta} + \frac{\delta \gamma_0 + \kappa^2}{\gamma_0^3} + \frac{p + 2q \gamma_0^4}{2 p \gamma_0} \right) \right)^{-1} .
    $$

\emph{Proof of Theorem \ref{thm_lambda_M}:}
Recall that 
    \begin{align*}
        H_{k,i} = (I - \rho_j s_j \bar{y}_j^T) H_{k, i-1} (I - \rho_j \bar{y}_j s_j^T) + \rho_j s_j s_j^T
    \end{align*}
    which is equivalent to 
    \begin{align*}
        H_{k, i} = H_{k, i-1} - \rho_j (s_j \bar{y}_j^T H_{k, i-1} + H_{k, i-1} \bar{y}_j s_j^T) + \rho_j^2 s_j \bar{y}_j^T H_{k, i-1} \bar{y}_j s_j^T + \rho_j s_j s_j^T .
    \end{align*}
    Taking the norm on $H_{k,i}$ gives
    \begin{align*}
        \Vert H_{k,i} \Vert &\le \Vert H_{k, i-1} \Vert + \frac{2 \Vert \bar{y}_j \Vert \Vert s_j \Vert}{s_j^T \bar{y}_j} \Vert H_{k,i-1} \Vert + \frac{s_j^T s_j}{s_j^T \bar{y}_j} \frac{y_j^T y_j}{s_j^T \bar{y}_j} \Vert H_{k, i-1} \Vert + \frac{s_j^T s_j}{s_j^T \bar{y}_j}    \nonumber   \\
        &\le \left(1 + \frac{2}{\delta \gamma_0 \kappa^2 q} + \left( \frac{1}{\delta \gamma_0 \kappa^2 q} \right)^2 \right) \Vert H_{k, i-1} \Vert + \frac{1}{\delta \gamma_0^2 \kappa^2 q}
    \end{align*}
    where we use Lemma \ref{lmm_H_k,i-pd}, Lemma \ref{lmm_useful-facts} and note that $w_j = \kappa^2 / \Gamma_j$, $q_j = q \Gamma_j^4$ with $\Gamma_j \ge \gamma_0$. Then by induction, we conclude
    \begin{equation}
        \lambda_M = \frac{1}{\delta \gamma_0^2 \kappa^2 q} \frac{a^p - 1}{a - 1}  \nonumber
    \end{equation}
    where $$a = 1 + \frac{2}{\delta \gamma_0 \kappa^2 q} + \left( \frac{1}{\delta \gamma_0 \kappa^2 q} \right)^2 .$$

\section{Code of the Experiments}
The datasets and the implementation of the experiments in Section \ref{sec_exp} can be found through the following link: \href{https://github.com/Starrskyy/clippedSQN}{https://github.com/Starrskyy/clippedSQN}.

\end{document}